\documentclass{article}

\usepackage[preprint]{neurips_2026}

\usepackage[utf8]{inputenc} 
\usepackage[T1]{fontenc}    
\usepackage{url}            
\usepackage{booktabs}       
\usepackage{amsfonts}       
\usepackage{nicefrac}       
\usepackage{microtype}      
\usepackage{xcolor}         

\usepackage{amsmath}
\usepackage{amssymb}
\usepackage{amsfonts}
\usepackage{graphicx}
\usepackage{bm}
\usepackage{amsthm}
\usepackage{algorithm}
\usepackage{algpseudocode}  
\usepackage{hyperref}
\usepackage{cleveref}
\usepackage{makecell}

\newtheoremstyle{paperplain}
  {3pt plus 1pt minus 1pt}
  {3pt plus 1pt minus 1pt}
  {\fontfamily{LibertinusSerif-TLF}\selectfont\itshape}
  {}
  {\fontfamily{LibertinusSerif-TLF}\selectfont\bfseries}
  {.}
  {0.5em}
  {}
\theoremstyle{paperplain}
\newtheorem{theorem}{Theorem}[section]
\newtheorem{proposition}[theorem]{Proposition}
\newtheorem{lemma}[theorem]{Lemma}
\newtheorem{corollary}[theorem]{Corollary}
\theoremstyle{definition}

\newtheorem{assumption}[theorem]{Assumption}
\theoremstyle{remark}
\newtheorem{remark}[theorem]{Remark}
\newcommand{\diff}[1]{\textcolor{red}{\tiny (#1)}}

\newcommand{\R}{\mathbb{R}}
\newcommand{\E}{\mathbb{E}}

\newcommand{\defeq}{\triangleq}
\newcommand{\hgrad}{\hat{\nabla}}

\newcommand{\va}{\bm{a}}
\newcommand{\vb}{\bm{b}}
\newcommand{\ve}{\bm{e}}
\newcommand{\vg}{\bm{g}}
\newcommand{\vu}{\bm{u}}
\newcommand{\vv}{\bm{v}}
\newcommand{\vz}{\bm{z}}

\newcommand{\mA}{\bm{A}}
\newcommand{\mB}{\bm{B}}
\newcommand{\mG}{\bm{G}}
\newcommand{\mI}{\bm{I}}
\newcommand{\mP}{\bm{P}}

\newcommand{\mW}{\bm{W}}

\newcommand{\din}{d_{\text{\normalfont in}}}
\newcommand{\dout}{d_{\text{\normalfont out}}}
\newcommand{\betagain}{\beta_{\text{\normalfont gain}}}

\newcommand{\vproj}{v_{\text{\normalfont proj}}}

\title{AR1-ZO: Topology-Aware Rank-1 Zeroth-Order Queries for High-Rank LoRA Fine-Tuning}

\author{%
  Ziye Chen\hspace{1.15em}
  Hongbin Lin\hspace{1.15em}
  Chenyu Zhang\hspace{1.15em}
  Xiangda Yan\hspace{1.15em}
  Yongjie Yang\hspace{1.15em}
  Yao Shu
}

\makeatletter
\g@addto@macro\@noticestring{\quad Correspondence to Yao Shu \texttt{<yaoshu@hkust-gz.edu.cn>}.}
\makeatother

\graphicspath{{workspace/figs/}}

\begin{document}

\maketitle

\begin{abstract}
  Zeroth-order (ZO) optimization enables large-language-model fine-tuning without storing backpropagation activations, while LoRA supplies compact trainable adapters. Combining them creates a rank paradox: increasing LoRA rank improves adapter capacity, but standard two-point ZO either perturbs a rank-dependent number of coordinates or, under atomwise updates, can make the finite-difference signal unobservable. This paper shows that the bottleneck is a measurement-topology problem rather than a need for an external subspace. LoRA already decomposes into matched rank-$1$ atoms, each a complete factor-coordinate block of dimension $\dout+\din$. Querying one atom per step keeps the stored adapter rank $r$ while removing $r$ from the single-query perturbation dimension. The naive atomwise query is still miscalibrated: if it inherits canonical LoRA scaling $\alpha/r$, the active finite-difference signal shrinks as $1/r$ and the active finite-difference signal-to-noise ratio (FD-SNR) as $1/r^2$, producing directional collapse under a fixed residual evaluation-noise floor. AR1-ZO pairs alternating rank-$1$ atom queries with topology-aware scaling $\gamma=\alpha r$, restoring rank-invariant active signal without auxiliary bases, activation hooks, curvature estimates, or extra forward queries. Theory proves atom minimality, rank-independent active query dimension, directional collapse and restoration, and the remaining rank dependence as an amortized coverage cost. Experiments on OPT and Qwen3 models validate the signal mechanism and show that AR1-ZO makes high-rank LoRA effective among matched-budget ZO methods under the standard two-forward-pass query budget.
\end{abstract}

\section{Introduction}
Fine-tuning large language models (LLMs) under tight memory budgets requires controlling both activation storage and trainable degrees of freedom. Zeroth-order (ZO) optimization, exemplified by MeZO \cite{mezo}, removes the activation-memory cost of backpropagation by estimating descent directions from forward-pass scalar differences. LoRA \cite{hu2022lora}, together with related low-rank adaptation variants such as QLoRA \cite{QLoRA}, DoRA \cite{DoRA}, and AdaLoRA \cite{adalora}, reduces the number of trainable coordinates by freezing pretrained weights and learning compact low-rank updates. Their combination is therefore natural, but it turns the LoRA rank into a double-edged quantity: the same rank that gives the adapter capacity can also determine what a two-point ZO oracle has to perturb and whether its scalar loss difference remains visible.

The usual explanation is the curse of dimensionality. Classical randomized gradient estimators incur variance that grows with the perturbation dimension \cite{nesterov2017random, liu2018zeroth}; if a ZO query perturbs all LoRA factors at once, the dimension grows with rank. Existing remedies reduce the effective dimension through block-coordinate perturbations \cite{MeZOBCD}, random subspaces \cite{yu2025subzero}, low-rank estimators \cite{LOZO,seung2025loren,sun2025tezo}, LoRA-manifold tangent updates \cite{song2026rozo}, activation-guided perturbations \cite{lin2026agzo}, and second-order reweighting \cite{zhao2025secondorder}. Recent low-rank and geometry-aware methods such as TeZO and RoZO further show that temporal low-rank structure and LoRA-manifold geometry can improve ZO fine-tuning, but they still leave open the atom-level measurement unit and scale of a two-forward-pass LoRA-ZO oracle. In parallel, LoRA scaling and optimization studies such as rsLoRA \cite{kala2023rslora}, RoRA \cite{liu2025rora}, and LoRA+ \cite{hayou2024lora} refine rank normalization or factor-wise dynamics under first-order joint updates. These lines establish that both query geometry and scaling matter, but they do not settle the measurement unit of black-box LoRA-ZO: \textbf{Can the rank-$1$ atoms already inside LoRA be used as complete two-forward-pass ZO query blocks, and what scaling keeps those atomwise measurements observable at high rank?} A fuller discussion of the related literature and novelty boundary appears in Appendix~\ref{sec:related_work}.

Our answer starts by separating three roles that rank often conflates: stored adapter capacity, single-query perturbation dimension, and finite-difference signal scale. The LoRA product can be written as a sum of rank-$1$ atoms, but the contribution is not this standard algebraic identity. The key object is the matched column-row pair $(\vb_k,\va_k)$, which is the finest complete LoRA-native factor block that can be exposed to a two-point oracle: perturbing only one side is incomplete, while grouping multiple atoms reintroduces a larger query dimension. An atomwise query therefore keeps the adapter rank $r$ in storage and in long-run coverage, but changes the instantaneous ZO problem from a rank-$r$ factor perturbation to a single complete atom perturbation. This removes the first bottleneck without invoking an auxiliary basis, activation hook, curvature estimate, or external subspace.

The same native query topology also reveals why the obvious atomwise optimizer fails. Canonical LoRA scaling uses the coefficient $\alpha/r$, calibrated for joint updates in which all $r$ atoms contribute together. Under a sequential rank-$1$ ZO measurement, only one atom differs between the positive and negative endpoints, so inheriting the joint coefficient suppresses the active finite-difference numerator by $1/r$. Residual evaluation noise is not reduced by the same factor, making the active FD-SNR decay as $1/r^2$ and eventually erasing directional fidelity before any learning-rate multiplier can act. We call this failure \textbf{directional collapse} induced by a \textbf{topology-scaling mismatch}. AR1-ZO resolves it by assigning the queried atom a rank-invariant active coefficient through the topology-aware calibration $\gamma=\alpha r$. This is not a new joint LoRA scaling rule; it is a measurement calibration for the sequential two-point oracle.

The resulting optimizer, Alternating Rank-1 Zeroth-Order optimization (AR1-ZO), cycles through complete rank-$1$ atoms while using the corrected active measurement scale. Each step still uses the standard two forward evaluations, and in transformer implementations the same active rank index can be shared across adapted matrices. The theoretical accounting then becomes precise: rank no longer enters the single-query perturbation dimension or the active FD-SNR; it remains as the amortized cost of revisiting all stored atoms. This yields three contributions.

\begingroup
\setlength{\leftmargini}{0.83em}
\setlength{\labelsep}{0.25em}
\begin{itemize}
\item \textbf{High rank need not be a single-query ZO dimension tax.} We formalize the matched atom $(\vb_k,\va_k)$ as the minimal complete LoRA-native query block and show that its active two-point estimator has perturbation dimension $\dout+\din$, not $r(\dout+\din)$, while preserving a rank-$r$ adapter.
\item \textbf{The high-rank failure of naive atomwise ZO is a measurement collapse.} We prove that inheriting the joint-update scale $\alpha/r$ makes active FD-SNR decay as $1/r^2$, derive the topology-aware correction $\gamma=\alpha r$ that restores rank-invariant active signal, and extend the accounting to rank-$m$ blocks and other joint-normalized factorizations.
\item \textbf{AR1-ZO makes the accounting operational under a pure black-box budget.} We instantiate alternating atom queries with two forward passes per step, prove the remaining rank dependence is an amortized coverage cost and characterize when aligned native atoms carry dense descent signal, then validate the mechanism with spectral-alignment diagnostics, rank-dependent active FD-SNR sweeps, downstream tasks, scaling-law controls, and efficiency measurements.
\end{itemize}
\endgroup

\begin{figure}[t]
  \centerline{\includegraphics[width=\linewidth]{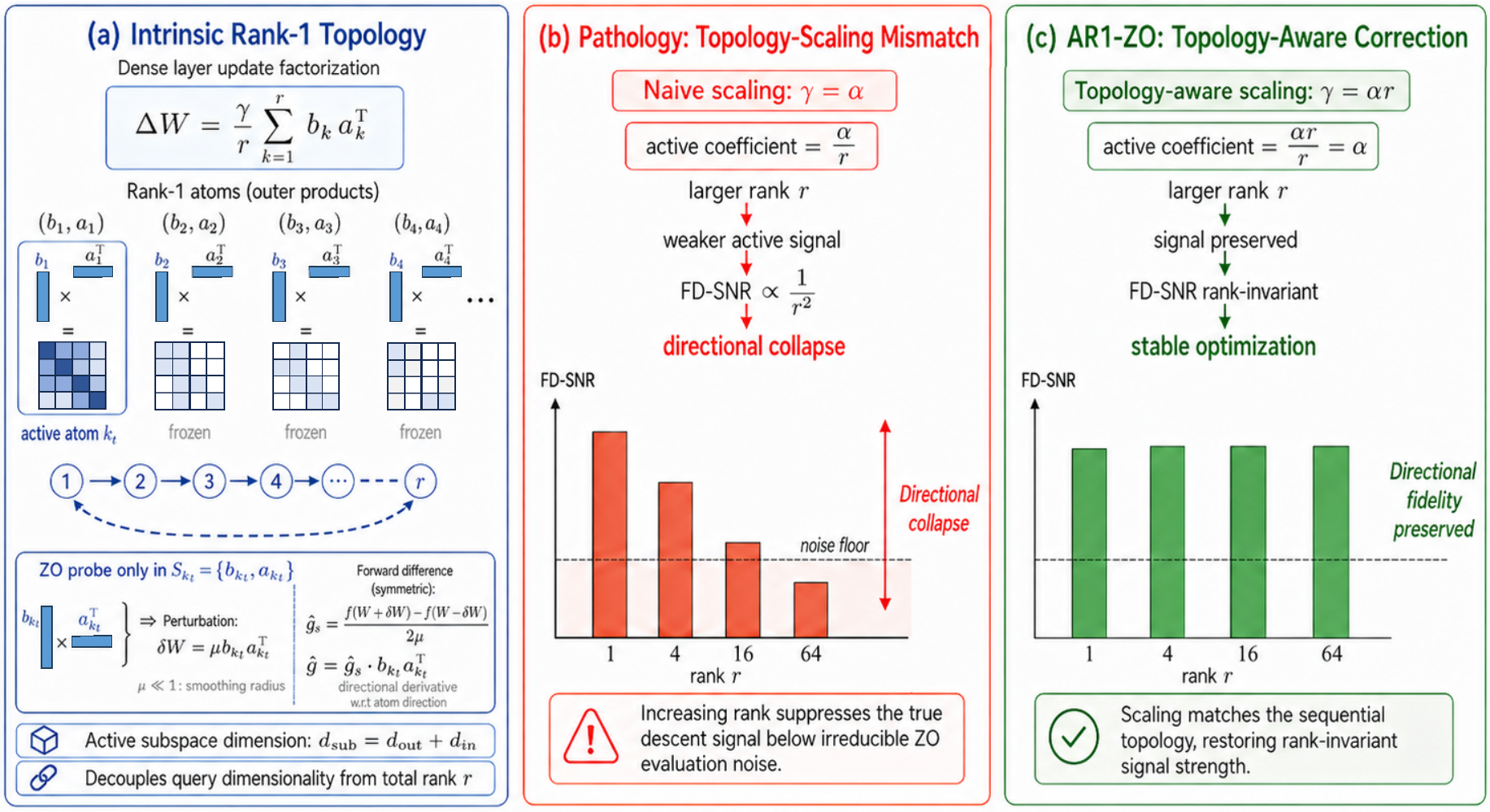}}
  \caption{The Optimization Pathology of LoRA-ZO and the AR1-ZO Solution.}
\vspace{-8pt}
\end{figure}

\section{Preliminaries: LoRA and Zeroth-Order Optimization}
\label{sec:preliminaries}

Applying ZO to LoRA is governed by three separable quantities: the trainable low-rank coordinates and output coefficient, the coordinates moved by one two-point query, and the evaluation noise that determines whether the scalar loss difference is observable. A black-box LoRA optimizer therefore faces two constraints: the queried coordinate set should not grow with rank, and the finite-difference numerator should not be scaled below the noise floor.

\textbf{LoRA parameterization and rank scaling.} LoRA turns dense adaptation into a factored low-rank parameterization. For a frozen dense layer $\mW_0\in\R^{\dout\times \din}$ and input activation $x$, LoRA injects a trainable rank-$r$ update through two factors $\mB\in\R^{\dout\times r}$ and $\mA\in\R^{r\times \din}$:
\begin{equation}
  h = \mW_0 x + \Delta\mW x,\qquad \Delta\mW = \frac{\gamma}{r}\mB\mA.
\label{eq:lora_update}
\end{equation}
The canonical choice $\gamma=\alpha$ gives coefficient $\alpha/r$ \cite{hu2022lora}, coupling rank to both adapter capacity and the output change induced by factor movement.

\textbf{Two-point ZO estimator and perturbation dimension.} ZO fine-tuning does not observe gradients of the LoRA factors directly; it observes how the loss changes under a random perturbation of the trainable coordinates. In the standard two-point Randomized Gradient Estimator (RGE) \cite{nesterov2017random}, a perturbation $\vz\sim\mathcal{N}(0,\mI_D)$ and smoothing radius $\mu$ produce
\begin{equation}
    \hgrad\mathcal{L}(\theta) = \frac{\mathcal{L}(\theta + \mu \vz) - \mathcal{L}(\theta - \mu \vz)}{2\mu}\,\vz.
    \label{eq:zo_grad}
\end{equation}
Here $D$ is the number of coordinates perturbed by one two-point query. Under standard $L$-smooth ZO analyses \cite{nesterov2017random, liu2018zeroth}, the mean squared error scales as $\mathcal{O}(D\|\nabla\mathcal{L}(\theta)\|^2 + L^2\mu^2 D^3)$. Perturbing all LoRA factors of one matrix sets $D=r(\dout+\din)$, so full-adapter ZO makes the rank chosen for capacity also control the estimator-variance dimension.

\textbf{Finite-difference signal and evaluation noise.} The dimension term explains why large perturbation spaces are difficult, but it does not fully determine whether a ZO query is useful. The estimator is formed from the clean finite-difference numerator
\begin{equation}
  \Delta_\mu(\theta;\vz)=\mathcal{L}(\theta+\mu\vz)-\mathcal{L}(\theta-\mu\vz),
\end{equation}
which carries the directional information in Eq.~\eqref{eq:zo_grad}. In stochastic LLM fine-tuning, paired evaluations on the same mini-batch can reduce data-sampling variance, but residual oracle variation from numerical kernels, stochastic components, or controlled noise injection still enters the observed numerator. We write the observed numerator as $\Delta_\mu(\theta;\vz)+\xi_+-\xi_-$, where $\xi_+$ and $\xi_-$ are the residual errors of the positive and negative evaluations and $\sigma_\xi^2$ denotes their variance scale after shared-batch control. A perturbation is informative only if its induced $\Delta_\mu$ remains resolvable against this noise. Dimension reduction therefore solves only half of the ZO-LoRA problem: a small coordinate block can still be useless if the coefficient $\gamma/r$ in Eq.~\eqref{eq:lora_update} suppresses its output effect before the loss difference is measured.

\section{AR1-ZO: Rank-1 Atom Queries with Topology-Aware Scaling}
\label{sec:method}

These constraints leave a narrow design target. Full-adapter ZO keeps the LoRA parameterization but perturbs all $r(\dout+\din)$ factor coordinates, so the rank chosen for capacity immediately becomes a variance multiplier. Existing dimension-reduction methods remove this tax with auxiliary bases, activation hooks, curvature side information, or manifold-level geometry. AR1-ZO instead changes the query topology of the existing LoRA factors. Its first step is to expose one complete LoRA atom to the two-point oracle, making the per-query coordinate set independent of $r$ while the stored adapter remains rank $r$. Its second step is to recalibrate the scale of that sequential measurement: once only one atom moves between the two endpoints, the joint LoRA coefficient $\alpha/r$ no longer reflects the amount of signal that the queried block must carry.

\subsection{Complete Rank-1 Atoms as Native ZO Query Blocks}

Fix one adapted matrix and write $\Delta\mW=(\gamma/r)\sum_{k=1}^{r}\vb_k\va_k^\top$, where $\vb_k$ is the $k$th column of $\mB$ and $\va_k^\top$ is the $k$th row of $\mA$. In first-order LoRA this decomposition is internal to joint-factor optimization, but in ZO the exposed coordinates define the search space. The matched atom $k$ is the pair $(\vb_k,\va_k)$: its first-order variations generate all local changes $\delta\vb_k\va_k^\top+\vb_k\delta\va_k^\top$. It is also the finest complete LoRA-native block, since perturbing one side is incomplete and grouping $m>1$ atoms raises the block dimension to $m(\dout+\din)$.

At iteration $t$, a schedule selects $k(t)\in\{1,\ldots,r\}$, exposes atom $k(t)$ to the oracle, and holds all other atoms fixed during the two endpoint evaluations. Let $q=\dout+\din$ be the number of stored coordinates in one atom, let $\vz\in\R^q$ be the active query direction, and let $\mP_t$ inject $\vz$ into the selected factor slots while zeroing all other LoRA coordinates. We write $\hgrad_{k(t)}\mathcal L$ for the resulting estimate of the gradient with respect to the active atom coordinates:
\begin{equation}
\hgrad_{k(t)} \mathcal{L} = \frac{\mathcal{L}(\theta + \mu \mP_t \vz)-\mathcal{L}(\theta - \mu \mP_t \vz)}{2\mu}\vz,
\label{eq:active_zo_estimator}
\end{equation}
using the same two loss evaluations as standard ZO. Equation~\eqref{eq:active_zo_estimator} changes the dimension term without changing the information model: the oracle still observes two scalar losses, and frozen atoms are shared by both endpoints. The stored adapter remains rank $r$, but one matrix query perturbs $\dout+\din$ rather than $r(\dout+\din)$ coordinates. Rank becomes a coverage index rather than an instantaneous variance multiplier. The active numerator, however, still inherits $\gamma/r$ from Eq.~\eqref{eq:lora_update}.

\subsection{Topology-Aware Scaling: \texorpdfstring{$\gamma=\alpha r$}{gamma=alpha r}}

Atomwise querying changes how the jointly normalized adapter is measured. The standard choice $\gamma=\alpha$ assumes a joint update in which all $r$ atoms contribute through the coefficient $\alpha/r$. Under Eq.~\eqref{eq:active_zo_estimator}, however, only one atom differs between the positive and negative endpoints, so that single atom must carry the observable finite-difference signal. To isolate this topology effect, fix the user-level active coefficient $\alpha$ and set
\begin{equation}
\gamma=\alpha r,\qquad \frac{\gamma}{r}=\alpha.
\label{eq:topology_aware_scale}
\end{equation}

With this calibration, the forward map seen by an active-atom query decomposes into an active term with rank-invariant coefficient and a frozen background shared by both endpoints:
\begin{equation}
\begin{aligned}
h_t(x)
&=
\mW_0x
+\underbrace{\alpha\,\vb_{k(t)}\va_{k(t)}^\top x\vphantom{\sum\nolimits_{j\neq k(t)}}}_{\text{\normalfont active atom}}
+\underbrace{\alpha\sum\nolimits_{j\neq k(t)}
  \vb_j\va_j^\top x}_{\text{\normalfont frozen background}}.
\end{aligned}
\label{eq:active_frozen_forward}
\end{equation}
Equations~\eqref{eq:topology_aware_scale}--\eqref{eq:active_frozen_forward} separate query dimension from measured signal scale at the model-output level. The active query dimension has already been reduced to $q$, and the frozen background can shape the current point, but it is identical in the two endpoint evaluations and therefore cannot rescue a suppressed active numerator. Under the naive fixed-scale choice $\gamma=\alpha$, the active coefficient would be $\alpha/r$, so the finite-difference signal amplitude is $\mathcal{O}(\alpha/r)$ and the signal power is $\mathcal{O}(\alpha^2/r^2)$ whenever the residual noise floor is rank-independent. Increasing rank can then increase representational capacity while making each atomwise measurement less reliable. A larger learning rate cannot repair this failure, because it multiplies the noisy estimate only after the finite-difference numerator has already lost directional information. Section~\ref{sec:theory_diagnosis} formalizes this collapse, and Section~\ref{sec:theory_repair} proves restoration under topology-aware scaling.

Rank-invariant measurement therefore requires rank-invariant $\gamma/r$. AR1-ZO uses the rank-$1$ measurement as the reference: $\alpha$ is the output coefficient assigned to the queried atom, giving $\gamma=\alpha r$. This calibration is topology-specific rather than a revision of joint first-order LoRA scaling; sequential ZO asks one atom, not the joint adapter, to carry the observable directional signal in a two-point measurement.

\subsection{Two-Forward-Pass AR1-ZO Update}
\label{sec:algorithm}

Algorithm~\ref{alg:ar1zo} exposes the finest complete LoRA-native atom block and measures it with the topology-aware coefficient. The cyclic schedule visits each atom once per cycle; each step samples the active atom's two factor directions, evaluates paired endpoints, and updates only the selected vectors.

\begin{algorithm}[H]
\caption{AR1-ZO update for one LoRA matrix}
\label{alg:ar1zo}
\begin{algorithmic}[1]
\Require Loss $\mathcal{L}$, rank $r$, active LoRA scale $\alpha$, smoothing radius $\mu$, step size $\eta_t$, atom schedule $k(t)=1+(t \bmod r)$.
\State Set $\gamma=\alpha r$ and choose the active atom $k(t)$.
\State Sample $\vz_b \sim \mathcal{N}(0,\mI_{\dout})$ and $\vz_a \sim \mathcal{N}(0,\mI_{\din})$; let $\vz_t=(\vz_b,\vz_a)$.
\State Evaluate $\ell_+ = \mathcal{L}(\theta_t+\mu \mP_t\vz_t)$ and $\ell_- = \mathcal{L}(\theta_t-\mu \mP_t\vz_t)$ on a paired mini-batch.
\State Form the active-atom estimator $\hat{\vg}_t=((\ell_+-\ell_-)/(2\mu))\,\vz_t$ and split it as $(\hat{\vg}_{b,t},\hat{\vg}_{a,t})$.
\State Update $\vb_{k(t)}\leftarrow\vb_{k(t)}-\eta_t\hat{\vg}_{b,t}$ and $\va_{k(t)}\leftarrow\va_{k(t)}-\eta_t\hat{\vg}_{a,t}$; keep all other atoms frozen.
\end{algorithmic}

\end{algorithm}
\vspace{-5pt}

In transformers, the same active rank index is shared across adapted matrices, so one pair of forward passes suffices. A full cycle covers all atoms; the price of rank is delayed coverage, not a larger per-step query budget.

\section{Theoretical Analysis}
\label{sec:theory}

AR1-ZO changes the role of LoRA rank in a two-point finite-difference
measurement. A rank-$r$ adapter still stores $r$ atoms, but a single
query perturbs only one matched atom and observes it through the LoRA
coefficient $\gamma/r$ in Eq.~\eqref{eq:lora_update}. Three quantities
govern the resulting estimator:
the dimension of the active random direction, the signal scale of the
active finite difference, and the time needed to revisit all stored
atoms. Matched atoms give a complete LoRA-native query block of
dimension $q=\dout+\din$ rather than $rq$; inherited joint scaling makes
the active numerator vanish with rank; topology-aware scaling restores
rank-invariant active signal, leaving rank only as an amortized coverage
cost. When an active atom aligns with dense gradient structure, the same
low-dimensional query can also carry dense descent signal. The
rank-invariance claim refers to the active finite-difference
measurement, conditional on the active-atom gradient regularity
formalized in the appendix; it is logically separate from the
forward-scale stability used by the convergence bound. Complete proofs
appear in Appendix~\ref{sec:appendix_proof}.

\subsection{Rank-\texorpdfstring{$1$}{1} Atoms as Minimal LoRA-Native Query Blocks}
\label{sec:theory_atoms}

The matched atom is the smallest complete block in the LoRA factors.
Perturbing only $\vb_k$ or only $\va_k$ misses first-order variations
of $\vb_k\va_k^\top$, while perturbing $(\vb_k,\va_k)$ keeps all local
factor directions of that atom inside the original LoRA
parameterization. AR1-ZO is therefore a LoRA-native atom query, not an
auxiliary subspace or one-sided factor shortcut.

\begin{lemma}[Completeness and minimality of matched rank-$1$ atoms]
\label{lem:atom_minimality}
\textit{
For nonzero factors, one-sided blocks are incomplete.  The matched
atom $(\vb_k,\va_k)$ is minimal complete, has dimension
$q=\dout+\din$, and extends to $mq$ for rank-$m$ blocks.
}
\end{lemma}

Minimality fixes the smallest complete LoRA-native coordinates
available to the oracle. The direction $\vz$ has only $q$ atom
coordinates and is lifted by the active projector; all other atoms are
identical at both endpoints. Under $L$-smoothness, projected RGE
moment bounds, and residual variance $\sigma_\xi^2$, the two-point
error bound counts the sampled atom dimension, not the stored
rank-$r$ adapter.

\begin{proposition}[Active-atom ZO dimension is independent of rank]
\label{prop:active_atom_dimension}
\textit{
For $k=k(t)$,
\[
\E\!\left[\|\hgrad_k\mathcal L-\nabla_k\mathcal L(\theta)\|^2\right]\lesssim q\|\nabla_k\mathcal L(\theta)\|^2+L^2\mu^2q^3+\frac{\sigma_\xi^2q}{\mu^2}.
\]
No term contains $rq$.
}
\end{proposition}

These statements separate stored capacity from instantaneous sampling dimension. Rank controls atom coverage, not the variance dimension of one query. The remaining issue is scale: endpoint losses see the active atom only after multiplication by $\gamma/r$.

\subsection{Diagnosis: Directional Collapse under Naive Scaling}
\label{sec:theory_diagnosis}

Sequential rank-$1$ querying exposes the active gradient through the
coefficient $\gamma/r$. Fix iteration $t$ and write $k=k(t)$. Let
$\vg_t=[\mG\va_k;\mG^\top\vb_k]$ be the unscaled active
factor-coordinate gradient, where $\mG=\nabla_{\mW}\mathcal L$ is the
dense matrix gradient; the gradient actually seen by the two-point
oracle is $(\gamma/r)\vg_t$. Since the oracle receives only a noisy
scalar projection of this vector, the active FD-SNR, denoted
$\mathsf{SNR}_t$ below, determines whether the direction remains
observable. The numerator of
Eq.~\eqref{eq:active_zo_estimator} gives
\begin{equation}
\frac{\Delta_\mu(\theta;\mP_t\vz)+\xi_+-\xi_-}{2\mu}
=\frac{\gamma}{r}\langle\vg_t,\vz\rangle+\frac{\xi_+-\xi_-}{2\mu}+\mathcal O(\mu^2),\quad
\mathsf{SNR}_t\asymp\frac{(\gamma/r)^2\mu^2\|\vg_t\|^2}{\sigma_\xi^2}.
\label{eq:active_fd_snr}
\end{equation}
This active FD-SNR belongs to the measured function difference, not the
subsequent optimizer step. With joint LoRA scaling $\gamma=\alpha$,
the clean numerator is divided by $r$ while residual endpoint noise is
not. High rank can therefore erase directional content before any
learning rate is applied.

\begin{theorem}[Directional Collapse under Naive Scaling]
\label{thm:directional_collapse}
\textit{
Naive $\gamma=\alpha$ gives active signal $\Theta(r^{-1})$ and
active FD-SNR $\mathsf{SNR}_t=\Theta(r^{-2})$.  With
$r_c=\Theta(\alpha\mu\|\vg_t\|/\sigma_\xi)$,
\[
\frac{r}{r_c}\to\infty\quad\Longrightarrow\quad \E\!\left[\cos\!\left(\hgrad_k\mathcal L,\nabla_k\mathcal L\right)\right]\to 0 .
\]
}
\end{theorem}

Directional collapse is a measurement failure, not a lack of adapter capacity. A low-dimensional active atom can still be drowned by residual noise once $\alpha/r$ suppresses the clean numerator. A larger learning rate only rescales the already noisy scalar difference, so useful high-rank updates require a query-time correction of the LoRA coefficient.

\subsection{Repair: Signal Restoration and Coverage Cost}
\label{sec:theory_repair}

Topology-aware scaling applies the correction before endpoint losses are queried. Setting $\gamma=\alpha r$ makes $\gamma/r=\alpha$, so the clean numerator in Eq.~\eqref{eq:active_fd_snr} no longer decays with rank. The FD-SNR claim uses full $\Theta(1)$ active-atom gradient regularity; the convergence claim uses stable represented-adapter scale to rule out output inflation and the upper-bound half $\|\vg_t\|=\mathcal O(1)$ to prevent hidden rank dependence in the descent inequality. The appendix states these complementary assumptions separately.

\begin{theorem}[Signal Restoration and Active-Atom Stationarity]
\label{thm:signal_restoration}
\textit{
With $\gamma=\alpha r$:
\textbf{(i)} active FD-SNR is rank-invariant,
$\mathsf{SNR}_t\asymp\alpha^2\mu^2\|\vg_t\|^2/\sigma_\xi^2=\Theta(1)$;
\textbf{(ii)} the projected RGE convergence bound holds,
\[
T^{-1}\sum\nolimits_{t=1}^{T}\E\|\nabla_{k(t)}\mathcal L(\theta_t)\|^2
\leq \mathcal O(T^{-1/2})+\mathcal O\!\left(\mu^4L^2q^2+\frac{\sigma_\xi^2q}{T^{1/2}\mu^2}\right).
\]
Part (i) uses active-atom gradient regularity; part (ii) uses
forward-scale stability of the represented adapter.
}
\end{theorem}

The bound depends on $q$, residual noise $\sigma_\xi$, and smoothing scale, but not on $rq$ or the naive $r^{-2}$ active FD-SNR decay. Rank reappears only through atom coverage:

\begin{corollary}[Coverage cost for full-adapter stationarity]
\label{cor:coverage_cost}
\textit{
For uniform atom sampling, or a cyclic schedule averaged over one
cycle,
\[
\E_{k(t)}\!\left[\|\nabla_{k(t)}\mathcal L(\theta)\|^2\right]=\frac{1}{r}\sum_{k=1}^{r}\|\nabla_k\mathcal L(\theta)\|^2 .
\]
Full-adapter stationarity therefore costs atom coverage, not a larger
per-query dimension.
}
\end{corollary}

The same calibration applies to wider blocks: a rank-$m$ active query trades higher per-step dimension $mq$ for a shorter coverage cycle.

\begin{corollary}[Extension to Rank-$m$ Blocks and Other Factorizations]
\label{cor:generalization}
\textit{
A rank-$m$ active block with target coefficient $\alpha$ uses
$\gamma=\alpha r/m$ and query dimension $mq$.  The same compensation
applies to joint-normalized multilinear factorizations under
sequential block queries.
}
\end{corollary}

\subsection{Structural Insight: Why Native Atoms Can Carry Dense Descent}
\label{sec:theory_insight}

A low-dimensional measurement is useful only when the measured atom carries dense descent signal. Let $\mG=\nabla_{\mW}\mathcal L=\sum_i\sigma_i\vu_i\vv_i^\top$, define $\beta=\cos^2(\vb_k,\vu_1)\cos^2(\va_k,\vv_1)$ as alignment with the top singular pair, and write $\vg_k=[\mG\va_k;\mG^\top\vb_k]$ for the unscaled active gradient. For $\Delta\mathcal L=\mathcal L(\theta_{t+1})-\mathcal L(\theta_t)$, the one-step bound first depends on active gradient energy and then relates that energy to spectral concentration and atom alignment.

\begin{theorem}[Conditional Progress Advantage of Aligned Rank-$1$ Atoms]
\label{thm:structural_advantage}
\textit{
If $\sigma_1^2\ge\rho\|\mG\|_F^2$, then
\[
\begin{aligned}
\E[-\Delta\mathcal L]
&\gtrsim
\eta\!\left(\frac{\gamma}{r}\right)^2\|\vg_k\|^2
-\frac{\eta^2L}{2}\E\!\left[\|\hgrad_k\mathcal L\|^2\right]
-\mathcal O(\eta\mu^4L^2q^2),\\
\|\vg_k\|^2
&\ge
\rho\beta\|\mG\|_F^2(\|\va_k\|^2+\|\vb_k\|^2).
\end{aligned}
\]
}
\end{theorem}

Alignment need not appear in every layer or at every training stage. When it is present, a single intrinsic rank-$1$ query carries a $\rho\beta$-scaled portion of dense descent signal while retaining active dimension $q$. The active-atom factor-norm regularity used in Section~\ref{sec:theory_repair} turns $(\|\va_{k(t)}\|^2+\|\vb_{k(t)}\|^2)$ into a $\Theta(1)$ multiplier; the signal advantage at active dimension $q$ is then $\Theta(\rho\beta\,d)$ when $\dout=\din=d$, with the exact ratio against generic random-block estimators given in Appendix~\ref{sec:appendix_proof}. Together with Theorem~\ref{thm:signal_restoration}, this separates the two ingredients behind faster practical convergence: $q$ controls the estimator variance, and $\rho\beta$ controls the useful signal. The experiments therefore measure spectral concentration and atom alignment, test the predicted active FD-SNR restoration under rank sweeps, and then evaluate whether restored high-rank capacity improves downstream performance under the same two-forward-pass budget.
\vspace{-8pt}
\begin{figure*}[h]
\begin{center}
\centerline{\includegraphics[width=\linewidth]{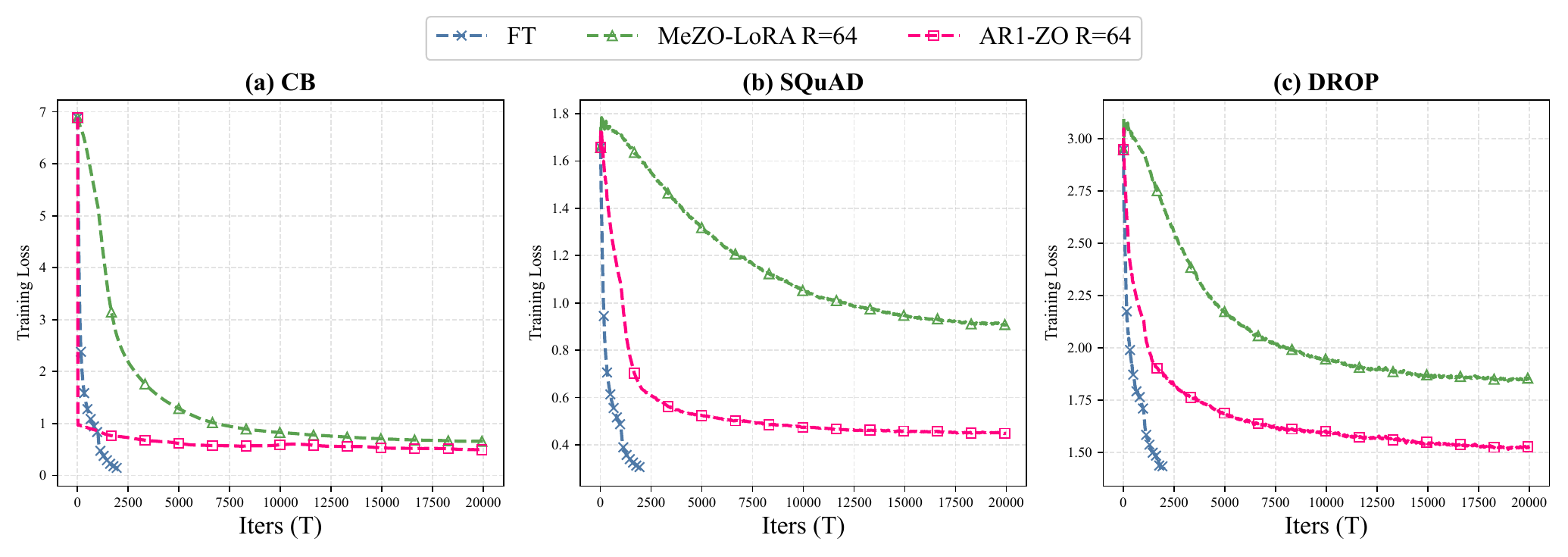}}
\caption{Training-loss trajectories on CB, SQuAD, and DROP. Under matched two-forward budgets, AR1-ZO descends smoothly among ZO methods across both classification and generation-style tasks; FT is shown as a first-order reference.}
\label{fig:loss_curves}
\end{center}
\vspace{-26pt}
\end{figure*}

\section{Experiments}
\label{sec:experiments}

The empirical study follows the theory's causal chain: Section~\ref{sec:dynamics} verifies end-to-end stability on representative CB/SQuAD/DROP loss curves; Section~\ref{sec:exp_mechanism} opens the box on the spectral premises of Theorem~\ref{thm:structural_advantage} and the rank-signal restoration of Theorems~\ref{thm:directional_collapse}--\ref{thm:signal_restoration}; Section~\ref{sec:exp_downstream} measures downstream payoff on OPT and Qwen3 families under matched two-forward budgets, with SQuAD controls in Appendix~\ref{sec:appendix_control_details}.

\subsection{Experimental Setup}
\label{sec:experimental_setup}

\textbf{Models, tasks, and baselines.} Downstream evaluation uses OPT-2.7B/13B as the primary family and Qwen3-1.7B/32B as modern-model confirmation on BoolQ, CB, COPA, WIC, SQuAD, and DROP. OPT experiments compare AR1-ZO against MeZO-LoRA, LOZO, ZO-Alt-naive, and FT (Adam); Qwen3 experiments isolate the scaling axis with MeZO, ZO-Alt-naive, and AR1-ZO. HIZOO public-implementation runs are reported in Appendix Table~\ref{tab:hizoo_public_code} as a reproducibility check under the same evaluation contract. AR1-ZO uses $\alpha=16$, $\gamma=\alpha r$, and all ZO methods use two forward evaluations per step.

\textbf{Diagnostic and ablation runs.} Mechanism/rank-signal diagnostics run on Qwen3-1.7B/COPA and SQuAD controls move to Appendix~\ref{sec:appendix_control_details}. Detailed splits, baseline tuning, hyperparameters, and profiling protocols are in Appendix~\ref{sec:appendix_experimental_details}.

\subsection{Training Dynamics and Convergence Speed}
\label{sec:dynamics}

With matched two-forward budgets fixed, we first ask whether AR1-ZO's atom accounting yields \emph{stable end-to-end optimization}, before opening the box in Section~\ref{sec:exp_mechanism}. Theorem~\ref{thm:signal_restoration} predicts stable active directions at dimension $q$ rather than $rq$; if so, loss curves should descend smoothly among matched-budget ZO methods on representative CB, SQuAD, and DROP runs, with FT (Adam) serving as a first-order reference rather than a rate target.

Figure~\ref{fig:loss_curves} shows that dimension reduction alone does not stabilize the finite-difference measurement across these tasks. AR1-ZO instead cycles atoms and reads each through the rank-invariant coefficient $\gamma/r$ (the projected-RGE second moment of Theorem~\ref{thm:signal_restoration}), giving the fastest matched-budget ZO curves and narrowing, but not closing, the FT gap. The full eight-task loss comparison appears in Appendix Fig.~\ref{fig:loss_curves_full}.

\subsection{Mechanism and Rank-Signal Validation}
\label{sec:exp_mechanism}
\label{sec:exp_diagnosis}

Stable loss curves are necessary but not sufficient, they could arise from any well-tuned ZO scheme. To attribute the gain to our mechanism, we test the two claims behind Theorems~\ref{thm:structural_advantage}--\ref{thm:signal_restoration} in turn: (i) the structural premises (spectral concentration, learned-atom alignment) hold along training, and (ii) the rank-scaling correction $\gamma=\alpha r$ restores active FD-SNR that naive scaling destroys.

\begin{figure*}[h]
\begin{center}
\centerline{\includegraphics[width=\linewidth]{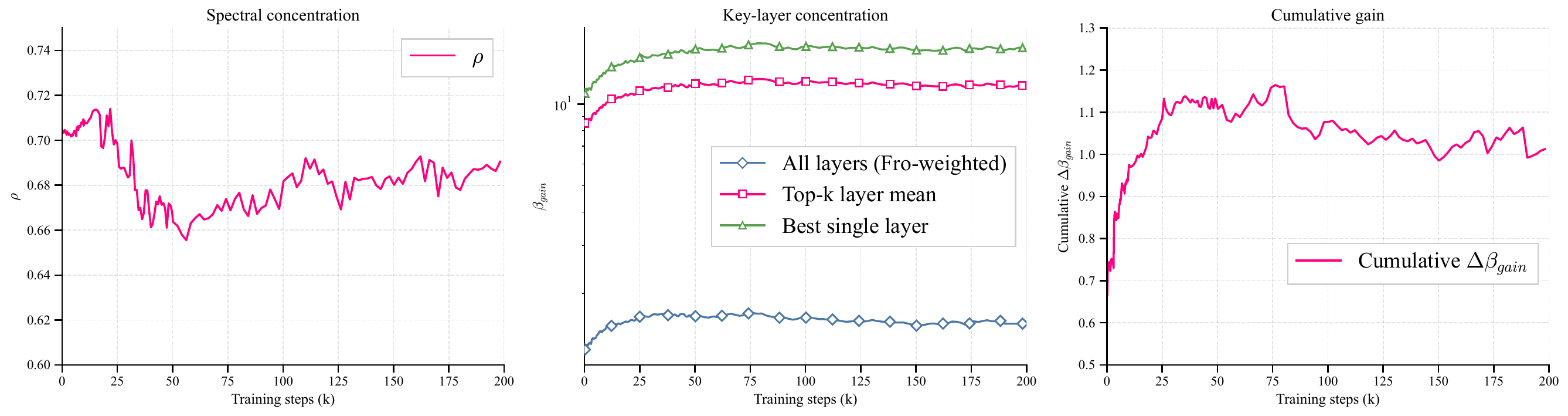}}
\caption{Mechanism validation over $2{\times}10^5$ steps. \textbf{Left:} Spectral concentration $\rho$ remains robust. \textbf{Middle:} Alignment gain concentrates heavily in high-impact layers. \textbf{Right:} Cumulative alignment gain remains strictly positive, confirming a net structural pull.}
\label{fig:mechanism}
\end{center}
\vspace{-12pt}
\end{figure*}

\begin{figure*}[h]
\begin{center}
\centerline{\includegraphics[width=\linewidth]{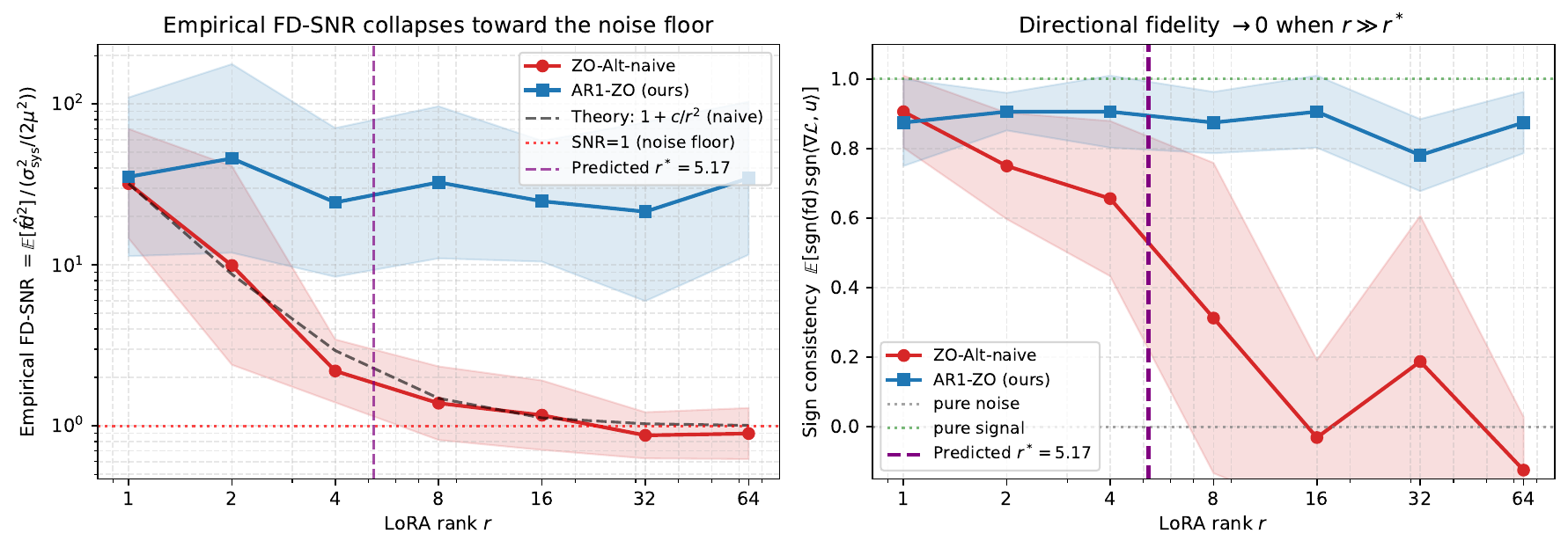}}
\caption{Signal restoration validation. Active FD-SNR and directional fidelity collapse under naive scaling but remain stable under AR1-ZO.}
\label{fig:collapse_diagnosis}
\end{center}
\vspace{-12pt}
\end{figure*}

\textbf{Claim (i): the gradient lives in a small structured subspace, and AR1-ZO aligns with it.} On Qwen3-1.7B/COPA over $2{\times}10^5$ steps, the three panels of Fig.~\ref{fig:mechanism} give complementary readings. Spectral concentration $\rho\in[0.65,0.75]$ \textit{(left)} means a few directions carry most of the gradient; top-$2$ layer alignment at $10\text{--}30\times$ the random null \textit{(middle)} means AR1-ZO's atoms hit those directions in the high-impact layers, not random ones; the overall cumulative alignment gain in $[2,7]$ \textit{(right)}, against zero for memoryless random projections, means the alignment compounds rather than cancels across steps. Appendix~\ref{sec:appendix_mechanism} reports the $\vb$- and $\va$-side cosine decomposition.

\textbf{Claim (ii): naive scaling starves the signal; $\gamma=\alpha r$ feeds it back.} Sweeping $r\in\{1,4,8,16,32,64\}$ on the same alternating topology, Fig.~\ref{fig:collapse_diagnosis} reads two health metrics: active FD-SNR (the signal-to-noise the optimizer can actually use) and directional fidelity (cosine to the true gradient). Naive scaling (ZO-Alt-naive) collapses both---FD-SNR drops below the optimization threshold at $r=64$ and directional fidelity decays to pure noise---exactly the topology-scaling mismatch of Theorem~\ref{thm:directional_collapse}. With $\gamma=\alpha r$, AR1-ZO holds both metrics steady across the entire sweep, so capacity can grow with $r$ without destroying what AR1-ZO measures (Theorem~\ref{thm:signal_restoration}).

\subsection{Capacity Scaling and Downstream Performance}
\label{sec:exp_downstream}

Premises and restored signal in Section~\ref{sec:exp_mechanism} are means, not ends; the payoff test is whether they translate into matched-budget downstream gains across model families and scales. We fix $r=64$ for LoRA-based methods and report a two-family two-scale matrix: OPT-2.7B/13B with full ZO baselines plus FT (Adam) as a first-order reference, and Qwen3-1.7B/32B with MeZO, ZO-Alt-naive, and AR1-ZO to isolate the scaling-law axis on a recent model family.

\textbf{OPT family.} Table~\ref{tab:full_results} shows AR1-ZO improves over MeZO-LoRA on all six OPT-2.7B tasks ($+0.6$--$+13.2$) and remains closest to FT (Adam) on COPA ($79.0$ vs $81.0$) and DROP ($27.8$ vs $31.4$); the same pattern of avoiding ZO-Alt-naive collapse holds at OPT-13B on CB/SQuAD/DROP per Fig.~\ref{fig:collapse_diagnosis}. \textbf{Qwen3 family.} AR1-ZO improves over MeZO by $+1.0$ to $+14.7$ on Qwen3-1.7B, and Qwen3-32B supplies the second modern-model scale. \textbf{Controls.} Appendix~\ref{sec:appendix_control_details} shows that sublinear scaling laws stagnate, learning-rate amplification cannot replace pre-oracle scaling, and AR1-ZO adds no systems overhead.

\begin{table*}[h]
\vspace{-10pt}
  \caption{\textbf{Main downstream results.} Accuracy (\%) for classification and F1 for SQuAD/DROP; differences are relative to MeZO-LoRA.}
  \label{tab:full_results}
  \centering
  \scriptsize
  \setlength{\tabcolsep}{2.2pt}
  \renewcommand{\arraystretch}{0.94}
  
  \begin{minipage}[t]{0.49\linewidth}
  \centering
  \textbf{OPT-2.7B}\\[-1pt]
  \begin{tabular}{lcccccc}
  \toprule
  Method & BoolQ & CB & Copa & WIC & SQuAD & DROP \\
  \midrule
  \makecell[l]{MeZO-LoRA\\($r{=}64$)} & 62.2 & 66.1 & 72.0 & 59.7 & 64.6 & 23.9 \\
  \midrule
  LOZO &
    \makecell{\textbf{65.6}\\\diff{+3.4}} &
    \makecell{67.9\\\diff{+1.8}} &
    \makecell{\textbf{80.0}\\\diff{+8.0}} &
    \makecell{54.4\\\diff{-5.3}} &
    \makecell{76.2\\\diff{+11.6}} &
    \makecell{23.5\\\diff{-0.4}} \\
  ZO-Alt-Naive &
    \makecell{55.5\\\diff{-6.7}} &
    \makecell{28.6\\\diff{-37.5}} &
    \makecell{71.0\\\diff{-1.0}} &
    \makecell{54.9\\\diff{-4.8}} &
    \makecell{19.2\\\diff{-45.4}} &
    \makecell{6.3\\\diff{-17.6}} \\
  \makecell[l]{\textbf{AR1-ZO}\\($r{=}64$)} &
    \makecell{64.6\\\diff{+2.4}} &
    \makecell{\textbf{69.6}\\\diff{+3.5}} &
    \makecell{79.0\\\diff{+7.0}} &
    \makecell{\textbf{60.3}\\\diff{+0.6}} &
    \makecell{\textbf{77.8}\\\diff{+13.2}} &
    \makecell{\textbf{27.8}\\\diff{+3.9}} \\
  \midrule
  \makecell[l]{\textit{FT}\\(Adam)} &
    \makecell{73.3\\\diff{+11.1}} &
    \makecell{85.7\\\diff{+19.6}} &
    \makecell{81.0\\\diff{+9.0}} &
    \makecell{63.8\\\diff{+4.1}} &
    \makecell{84.2\\\diff{+19.6}} &
    \makecell{31.4\\\diff{+7.5}} \\
  \bottomrule
  \end{tabular}
  \end{minipage}
  \hfill
\begin{minipage}[t]{0.49\linewidth}
\centering
\textbf{OPT-13B}\\[-1pt]
\begin{tabular}{lcccccc}
\toprule
Method & BoolQ & CB & Copa & WIC & SQuAD & DROP \\
\midrule
\makecell[l]{MeZO-LoRA\\($r{=}64$)} & 64.8 & 66.1 & 84.0 & 54.9 & 78.7 & 30.1 \\
\midrule
LOZO &
   \makecell{68.1\\\diff{+3.3}} &
   \makecell{\textbf{69.6}\\\diff{+3.5}} &
   \makecell{\textbf{86.0}\\\diff{+2.0}} &
   \makecell{56.9\\\diff{+2.0}} &
   \makecell{\textbf{81.4}\\\diff{+2.7}} &
   \makecell{30.6\\\diff{+0.5}} \\
ZO-Alt-Naive &
   \makecell{56.9\\\diff{-7.9}} &
   \makecell{33.7\\\diff{-32.4}} &
   \makecell{75.0\\\diff{-9.0}} &
   \makecell{55.8\\\diff{+0.9}} &
   \makecell{20.7\\\diff{-58.0}} &
   \makecell{8.1\\\diff{-22.0}} \\
\makecell[l]{\textbf{AR1-ZO}\\($r{=}64$)} &
   \makecell{\textbf{66.7}\\\diff{+1.9}} &
   \makecell{\textbf{69.6}\\\diff{+3.5}} &
   \makecell{\textbf{86.0}\\\diff{+2.0}} &
   \makecell{\textbf{57.5}\\\diff{+2.6}} &
   \makecell{79.4\\\diff{+0.7}} &
   \makecell{\textbf{30.8}\\\diff{+0.7}} \\
\midrule
\makecell[l]{\textit{FT}\\(Adam)} &
   \makecell{76.9\\\diff{+12.1}} &
   \makecell{84.1\\\diff{+18.0}} &
   \makecell{79.0\\\diff{-5.0}} &
   \makecell{70.1\\\diff{+15.2}} &
   \makecell{84.9\\\diff{+6.2}} &
   \makecell{31.3\\\diff{+1.2}} \\
\bottomrule
\end{tabular}
\end{minipage}
  
  \vspace{3pt}
  
  \begin{minipage}[t]{0.49\linewidth}
  \centering
  \textbf{Qwen3-1.7B}\\[-1pt]
  \begin{tabular}{lcccccc}
  \toprule
  Method & BoolQ & CB & Copa & WIC & SQuAD & DROP \\
  \midrule
  \makecell[l]{MeZO-LoRA\\($r{=}64$)} & 74.7 & 73.2 & 71.0 & 53.3 & 62.0 & 40.1 \\
  \midrule
  ZO-Alt-Naive &
    \makecell{68.5\\\diff{-6.2}} &
    \makecell{41.1\\\diff{-32.1}} &
    \makecell{71.0\\\diff{+0.0}} &
    \makecell{49.8\\\diff{-3.5}} &
    \makecell{17.3\\\diff{-44.7}} &
    \makecell{10.2\\\diff{-29.9}} \\
  \makecell[l]{\textbf{AR1-ZO}\\($r{=}64$)} &
    \makecell{\textbf{78.0}\\\diff{+3.3}} &
    \makecell{\textbf{87.5}\\\diff{+14.3}} &
    \makecell{\textbf{72.0}\\\diff{+1.0}} &
    \makecell{\textbf{56.6}\\\diff{+3.3}} &
    \makecell{\textbf{76.7}\\\diff{+14.7}} &
    \makecell{\textbf{42.9}\\\diff{+2.8}} \\
  \bottomrule
  \end{tabular}
  \end{minipage}
  \hfill
  \begin{minipage}[t]{0.49\linewidth}
    \centering
    \textbf{Qwen3-32B}\\[-1pt]
    \begin{tabular}{lcccccc}
    \toprule
    Method & BoolQ & CB & Copa & WIC & SQuAD & DROP \\
    \midrule
    \makecell[l]{MeZO-LoRA\\($r{=}64$)} & 83.6 & 89.3 & 77.0 & 68.3 & 73.2 & \textbf{58.5} \\
    \midrule
    ZO-Alt-Naive &
       \makecell{71.6\\\diff{-12.0}} &
       \makecell{45.4\\\diff{-43.9}} &
       \makecell{73.0\\\diff{-4.0}} &
       \makecell{63.3\\\diff{-5.0}} &
       \makecell{21.1\\\diff{-52.1}} &
       \makecell{15.1\\\diff{-43.4}} \\
    \makecell[l]{\textbf{AR1-ZO}\\($r{=}64$)} &
       \makecell{\textbf{85.6}\\\diff{+2.0}} &
       \makecell{\textbf{90.4}\\\diff{+1.1}} &
       \makecell{\textbf{88.0}\\\diff{+11.0}} &
       \makecell{\textbf{71.5}\\\diff{+3.2}} &
       \makecell{\textbf{80.6}\\\diff{+7.4}} &
       \makecell{52.2\\\diff{-6.3}} \\
    \bottomrule
    \end{tabular}
    \end{minipage}
  
  \vspace{-12pt}
  \end{table*}

\section{Conclusion}
\label{sec:conclusion}

We identified and resolved a topology-scaling mismatch in zeroth-order LoRA fine-tuning: inheriting $\alpha/r$ under sequential atomwise queries drives the active signal below the finite-difference noise floor. \textbf{AR1-ZO} pairs alternating rank-$1$ probes with $\gamma=\alpha r$, making high-rank LoRA usable under pure black-box constraints. Rank leaves the single-query dimension and active FD-SNR, remaining only as amortized atom coverage.

\paragraph{Limitations and Future Work.} Theorem~\ref{thm:signal_restoration} is active-atom stationarity, so full-adapter stationarity still pays the coverage factor in Corollary~\ref{cor:coverage_cost}. Experiments cover OPT (2.7B/13B) and Qwen3 (1.7B/32B); larger models, free-form generation, and tensorized adapters remain natural extensions.

\bibliographystyle{abbrv}
\bibliography{workspace/reference}

\appendix
\section{Related Work}
\label{sec:related_work}

Optimizing LLMs under strict memory constraints has driven the convergence of LoRA-style adaptation and zeroth-order (ZO) optimization. The relevant literature is therefore not a single lineage but the intersection of three threads: LoRA's rank-factor structure, dimensionality reduction in ZO, and LoRA scaling dynamics. AR1-ZO is defined by two specific claims rather than by the broad fact that LoRA is low-rank. First, it uses each matched column-row rank-$1$ atom as the complete query block exposed to a two-point ZO oracle. Second, it shows that this sequential query topology changes the meaning of LoRA's scale and requires the topology-aware correction $\gamma=\alpha r$ to keep the finite-difference signal observable.

\paragraph{LoRA and Algebraic Decomposition.}
LoRA freezes pre-trained weights and injects trainable low-rank modules. Its update $\Delta \mW = \mB\mA$ can be written as an additive ensemble of rank-$1$ atoms, and structural refinements such as AdaLoRA~\cite{adalora} make related rank components explicit. Because these components are additively separable, they invite alternating and blockwise optimization: RoLoRA~\cite{RoLoRA} alternates $\mA$/$\mB$ to stabilize federated FO training, and MeZO-BCD~\cite{MeZOBCD} updates parameter subsets to reduce effective ZO dimensionality. These works establish that LoRA factors and parameter blocks can be exposed selectively, but they do not use the matched rank-$1$ atom $(\vb_k,\va_k)$ as the minimal complete two-point ZO query block. The distinction matters because a one-sided factor update does not span all local variations of an atom, while larger blocks reintroduce rank-dependent query dimension.
\textit{Our positioning:} the atom contribution is not the identity $\mB\mA=\sum_k \vb_k\va_k^\top$ itself; it is the use of the complete rank-$1$ atom as the LoRA-native measurement unit that reduces the per-query dimension without reducing stored adapter rank.

\paragraph{Subspace and Low-Rank Zeroth-Order Optimization.}
Standard ZO methods such as MeZO~\cite{mezo} suffer variance proportional to the full parameter dimension, prompting several subspace remedies. SubZero~\cite{yu2025subzero} restricts perturbations to random low-dimensional subspaces, while LOZO~\cite{LOZO} builds low-rank ZO estimators from sampled low-rank perturbations. TeZO~\cite{sun2025tezo} extends this low-rank viewpoint across training time by constructing perturbations through a temporal tensor decomposition, and LOREN~\cite{seung2025loren} and HIZOO~\cite{zhao2025secondorder} improve search directions with curvature or Hessian-informed statistics. AGZO~\cite{lin2026agzo} constructs activation-guided perturbation subspaces from forward-pass information. RoZO~\cite{song2026rozo} is the closest LoRA-aware geometry method: it treats LoRA adapters as a low-rank manifold and constrains ZO updates to tangent spaces with Riemannian tools.
\textit{Our positioning:} AR1-ZO is complementary to these subspace, low-rank, and geometry methods. It does not sample an auxiliary basis, use activation hooks, estimate curvature, decompose perturbations over time, or optimize on the full LoRA manifold tangent space. Instead, it queries the structural rank-$1$ atoms already present in the parameterization. TeZO shows that temporal structure can reduce ZO perturbation overhead, and RoZO shows that LoRA geometry is useful for ZO; AR1-ZO studies the finer measurement topology obtained when one atom is queried at a time, and identifies the scaling pathology that appears only under this sequential atomwise oracle.

\paragraph{Optimization Topology and Scaling Dynamics.}
LoRA scaling laws such as the canonical $\alpha/r$~\cite{hu2022lora}, rsLoRA's $\alpha/\sqrt{r}$~\cite{kala2023rslora}, and RoRA's related rank-reliability correction~\cite{liu2025rora} study how adapter magnitude should vary with rank during ordinary LoRA fine-tuning. LoRA+~\cite{hayou2024lora} addresses a different but related optimization imbalance by assigning different learning rates to $\mA$ and $\mB$. These works are important because they show that LoRA scaling is not a harmless implementation detail. Their topology, however, is still joint first-order adaptation: all rank components participate in the update signal together.
\textit{Our positioning:} AR1-ZO's scaling claim is not merely that a larger scalar can improve high-rank LoRA. The claim is that when a two-point ZO oracle perturbs only one rank-$1$ atom, the inherited joint coefficient $\alpha/r$ multiplies the clean finite-difference numerator while the residual evaluation noise is unchanged. This makes the active finite-difference signal-to-noise ratio (FD-SNR) decay as $1/r^2$. The correction $\gamma=\alpha r$ is therefore a topology-aware measurement calibration: it assigns the active atom a rank-invariant observable signal while keeping the remaining atoms fixed across the paired evaluations.

\section{Experimental Details}
\label{sec:appendix_experimental_details}

\paragraph{Comparison roles and contract.}
The empirical study separates four evidence layers, each carried by its own subsection in the main text. (i) Trajectory-level convergence and stability are measured directly on the primary downstream family (OPT-2.7B). (ii) Mechanism diagnostics on COPA test the theorem-level quantities used by the analysis: spectral concentration $\rho$, atom alignment gain $\betagain$, active FD-SNR, and directional fidelity. (iii) Downstream tables test whether the restored high-rank signal remains useful under matched two-forward-pass budgets, on the OPT family (2.7B and 13B) as the primary axis and on Qwen3 (1.7B and 32B) as a modern-model confirmation. (iv) SQuAD ablations isolate scaling law choice, rule out learning-rate compensation, and measure step time and peak memory under the same profiling protocol. Keeping these four axes apart prevents downstream accuracy from being treated as the only evidence and lets each theorem be tested on the quantity it actually predicts.

\paragraph{Tasks, splits, and models.}
The main downstream evaluation uses OPT-2.7B, OPT-13B, Qwen3-1.7B, and Qwen3-32B on BoolQ, CB, COPA, WIC, SQuAD, and DROP. The full loss-trajectory diagnostic additionally includes WSC and ReCoRD, reported only in Appendix Fig.~\ref{fig:loss_curves_full}. The explicit split sizes used in the main text are CB $(250/56/250)$, COPA $(400/100/500)$, and SQuAD $(1000/500/1000)$ for the ablation setting. Mechanism and rank-diagnostic experiments use COPA, while scaling-law and efficiency controls use SQuAD with Qwen3-1.7B.

\paragraph{Baselines and query matching.}
On the OPT family AR1-ZO is compared against unstructured ZO (MeZO, MeZO-LoRA), structured ZO (LOZO), the diagnostic ZO-Alt-naive baseline, and Adam full fine-tuning (FT) as a first-order upper bound from the MeZO comparison setting. HIZOO public-implementation runs are reported separately in Table~\ref{tab:hizoo_public_code} as a reproducibility check because the released code showed unstable performance under the matched LoRA-ZO evaluation contract. On the Qwen3 family the comparison narrows to MeZO, ZO-Alt-naive, and AR1-ZO, isolating the scaling-law axis from auxiliary-basis or curvature-side-information effects. ZO-Alt-naive uses the same alternating rank-$1$ topology as AR1-ZO but keeps the canonical joint-update scaling, making it the direct control for the topology-scaling correction. All ZO methods are matched to two forward evaluations per optimization step.

\begin{table}[t]
\caption{\textbf{HIZOO public-code reproducibility check.} Results are from the released HIZOO implementation under the same OPT LoRA-ZO evaluation contract used in the main table; differences are relative to MeZO-LoRA.}
\label{tab:hizoo_public_code}
\centering
\scriptsize
\setlength{\tabcolsep}{3.2pt}
\begin{tabular}{lcccccc}
\toprule
Model & BoolQ & CB & COPA & WIC & SQuAD & DROP \\
\midrule
OPT-2.7B &
\makecell{60.8\\\diff{-1.4}} &
\makecell{55.4\\\diff{-10.7}} &
\makecell{60.0\\\diff{-12.0}} &
\makecell{45.6\\\diff{-14.1}} &
\makecell{0.0\\\diff{-64.6}} &
\makecell{0.1\\\diff{-23.8}} \\
OPT-13B &
\makecell{63.1\\\diff{-1.7}} &
\makecell{41.0\\\diff{-25.1}} &
\makecell{52.0\\\diff{-32.0}} &
\makecell{50.0\\\diff{-4.9}} &
\makecell{0.0\\\diff{-78.7}} &
\makecell{0.1\\\diff{-30.0}} \\
\bottomrule
\end{tabular}
\end{table}

\paragraph{Diagnostic protocols.}
For the mechanism validation in Fig.~\ref{fig:mechanism}, Qwen3-1.7B is fine-tuned with LoRA rank $r=64$ on COPA for $2{\times}10^5$ steps. The full $400$-example training split is used as a fixed diagnostic probe set, and spectral/alignment statistics are aggregated across the $28$ $\vproj$ matrices with Frobenius weighting. For the rank sweep in Fig.~\ref{fig:collapse_diagnosis}, LoRA rank varies over $r\in\{1,4,8,16,32,64\}$, the function-evaluation noise $\sigma_\xi$ is isolated through controlled variance injection, and the finite-difference probes use unit-normalized Gaussian perturbations.

\paragraph{Ablation and efficiency protocols.}
The SQuAD ablations in Fig.~\ref{fig:ablation_scaling} use Qwen3-1.7B, LoRA rank $r=64$, $\alpha=16$ unless otherwise noted, batch size $16$, smoothing radius $\mu=10^{-3}$, and seed $42$. The scaling-law sweep compares active output coefficients $\gamma/r\in\{\alpha/r,\alpha/\sqrt{r},\alpha\}$, equivalently $\gamma\in\{\alpha,\alpha\sqrt r,\alpha r\}$, over $2{\times}10^4$ ZO steps with the alternating topology fixed. The learning-rate control places ZO-Alt-naive in its constant-$\gamma$ worst case with $\alpha=1$, $r=64$, and $\gamma=1$, then sweeps $\eta\in\{10^{-5},10^{-4},10^{-3}\}$. Runtime and memory are profiled over $500$ ZO steps; LOZO uses lazy refresh $50$.

\begin{figure*}[h]
\begin{center}
\centerline{\includegraphics[width=\linewidth]{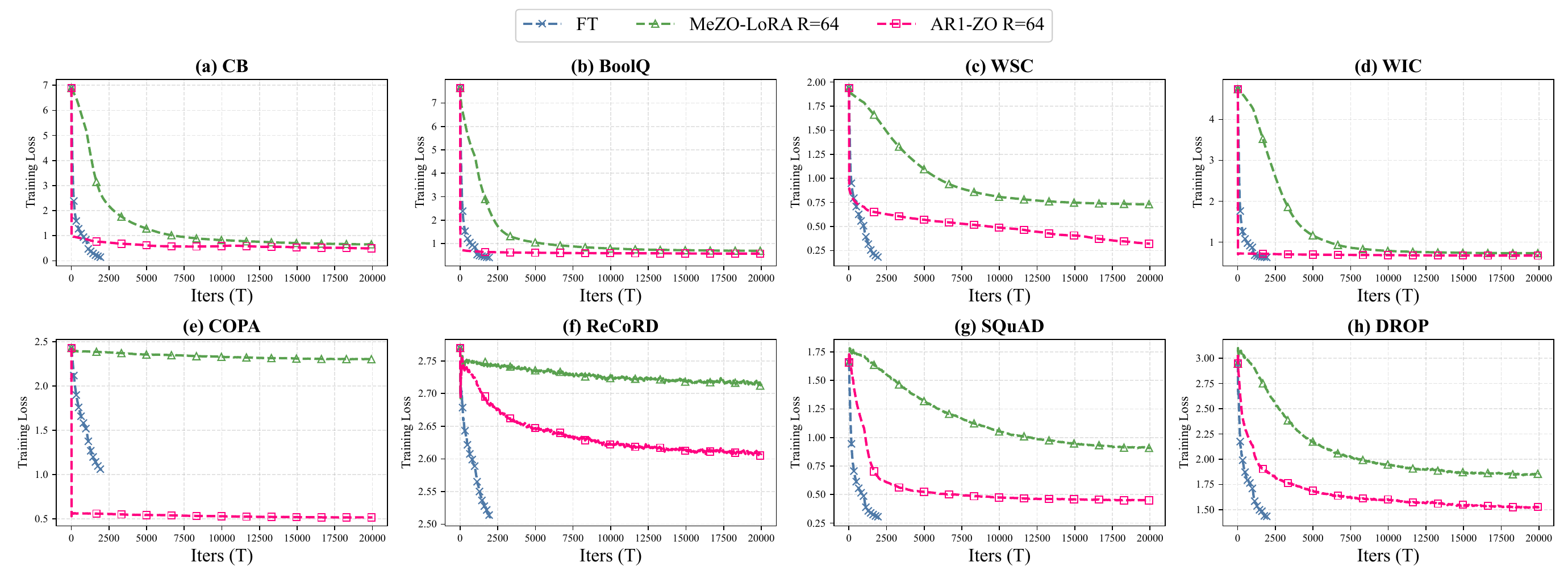}}
\caption{Full eight-task training-loss comparison. The appendix reports the complete loss trajectories on CB, BoolQ, WSC, WIC, COPA, ReCoRD, SQuAD, and DROP; the main text shows a compact three-task subset for readability.}
\label{fig:loss_curves_full}
\end{center}
\end{figure*}

\subsection{Mechanism Validation Details}
\label{sec:appendix_mechanism}

This section expands Section~\ref{sec:exp_mechanism} on the mechanism diagnostics summarized in Fig.~\ref{fig:mechanism}.

\paragraph{Layer anisotropy.}
The Frobenius-weighted network average of the alignment gain $\betagain$ is modest ($1\text{--}3\times$ over the analytical random null), but the dispersion across the $28$ $\vproj$ matrices is heavy-tailed: the top-$2$ layer mean reaches $10\text{--}30\times$ and the best single layer reaches $20\text{--}100\times$. The structural advantage of intrinsic atoms therefore concentrates in a task-dependent subset of high-impact layers rather than appearing as a uniform layerwise improvement, consistent with the conditional nature of Theorem~\ref{thm:structural_advantage}.

\paragraph{Per-cycle saturation and the cumulative diagnostic.}
Each atom saturates within its first rank cycle, so per-window $\betagain$ increments oscillate around zero in the mid-to-late training regime and do not by themselves expose the directional pull. We therefore track the cumulative quantity $\betagain(t)-\betagain(t_0)$ against a fixed first-end baseline $t_0$ chosen at the end of the first complete rank cycle. This trajectory is strictly positive after $\sim\!10^4$ steps and stabilizes in $[2,7]$, while a memoryless random projection (e.g.\ LOZO's sampled low-rank perturbation) has zero expected cumulative gain under the same diagnostic. The diagnostic therefore distinguishes intrinsic-atom locking from random-block sampling at the trajectory level rather than the per-step level.

\begin{figure*}[h]
\begin{center}
\centerline{\includegraphics[width=0.7\linewidth]{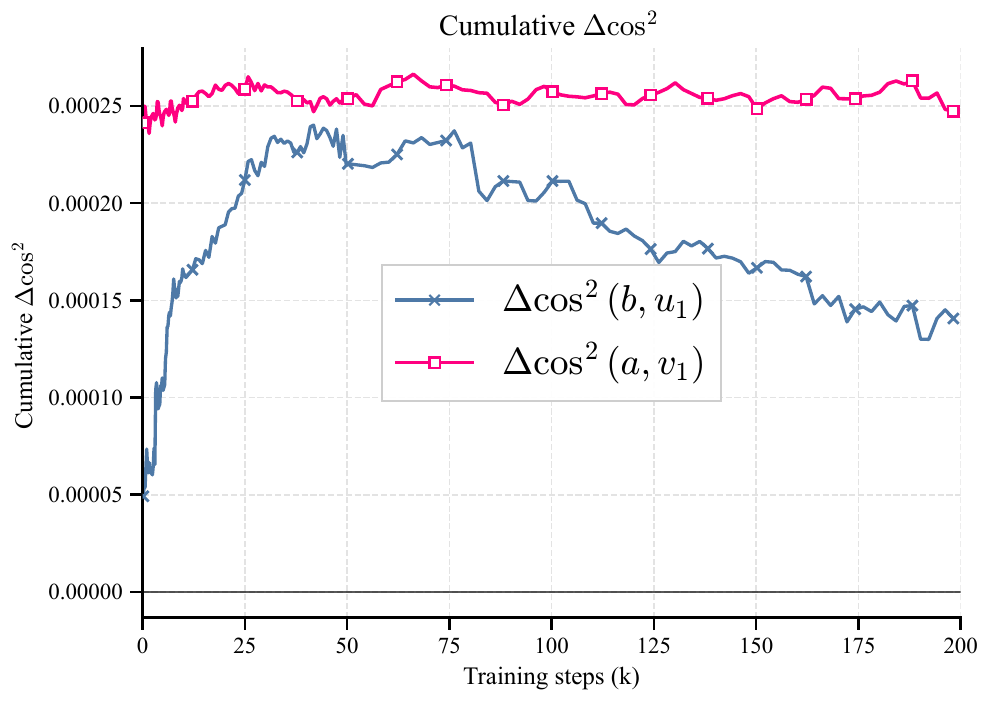}}
\caption{Side-specific cumulative alignment decomposition. The output-side increment $\Delta\cos^2(\vb,\vu_1)$ contributes most of the positive cumulative alignment, while the input-side increment $\Delta\cos^2(\va,\vv_1)$ remains smaller but positive.}
\label{fig:mechanism_cos2}
\end{center}
\end{figure*}

\paragraph{$\vb$- and $\va$-side decomposition.}
Figure~\ref{fig:mechanism_cos2} decomposes $\beta=\cos^2(\vb,\vu_1)\cos^2(\va,\vv_1)$ and shows that the cumulative gain is driven mainly by the output-side cosine $\cos^2(\vb,\vu_1)$, with a smaller positive contribution from the input-side $\cos^2(\va,\vv_1)$. This asymmetry is consistent with the observation that LoRA's $\mB$ factor (initialized at zero and grown by the optimizer) carries the primary alignment dynamics, while $\mA$ (initialized random) provides a near-isotropic input-side contribution.

\subsection{Control Ablation Details}
\label{sec:appendix_control_details}

This section expands the controls summarized in Fig.~\ref{fig:ablation_scaling}.

\begin{figure*}[t]
\begin{center}
\centerline{\includegraphics[width=\linewidth]{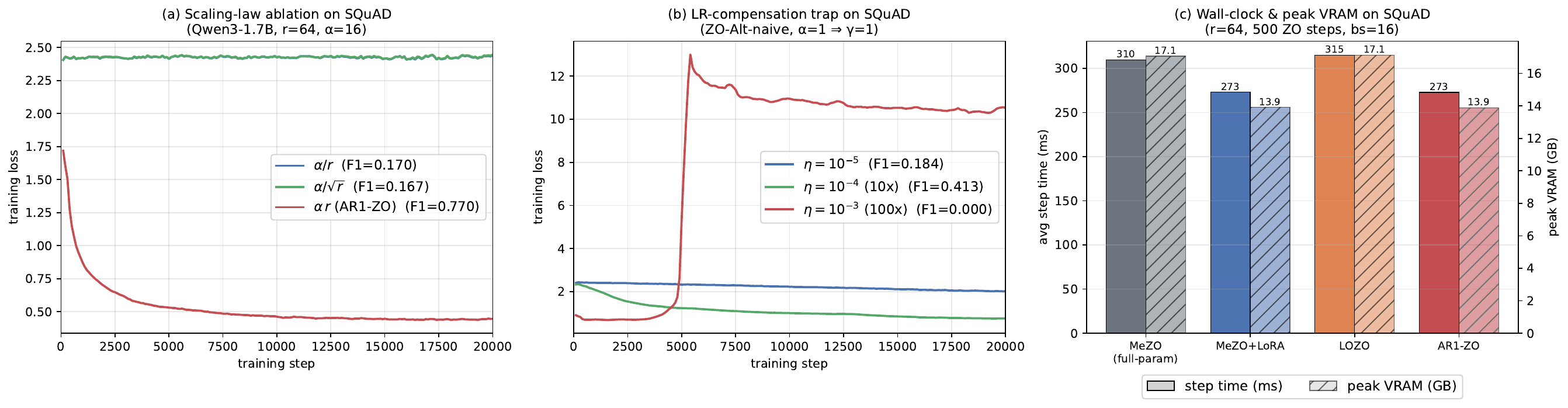}}
\caption{Ablations and computational efficiency on SQuAD. \textbf{Left:} Only $\gamma{=}\alpha r$ escapes directional collapse (F1 $0.770$ vs.\ $\sim\!0.17$). \textbf{Middle:} Amplifying $\eta$ on ZO-Alt-naive partially recovers ($10\times$, F1 $0.413$) before destabilizing ($100\times$, F1 $0.000$). \textbf{Right:} AR1-ZO matches MeZO+LoRA exactly ($273$\,ms / $13.9$\,GB), while LOZO incurs $\sim\!15\%$ slower steps and $\sim\!23\%$ higher peak VRAM.}
\label{fig:ablation_scaling}
\end{center}
\end{figure*}

\paragraph{Scaling-law sweep.}
Both the canonical active coefficient $\alpha/r$ (the $\gamma=\alpha$ regime of Theorem~\ref{thm:directional_collapse}) and the milder $\alpha/\sqrt r$ stagnate at the noise floor (F1 $\approx 0.17$, indistinguishable from random). This confirms that any sub-linear $\gamma(r)$ leaves rank-dependent finite-difference attenuation. In contrast, the rank-invariant active coefficient $\alpha$ implemented by $\gamma=\alpha r$ reaches F1 $=0.770$ with monotone, non-oscillatory descent, in agreement with the restoration predicted by Theorem~\ref{thm:signal_restoration}.

\paragraph{Learning-rate compensation trap.}
The LR-compensation control tests the natural objection that $\gamma$ is interchangeable with $\eta$. It is not: a $10\times$ learning-rate increase on ZO-Alt-naive raises F1 only to $0.413$ before training destabilizes around step $5{,}000$, while $100\times$ diverges outright. The asymmetry is structural. The scale $\gamma$ changes the active atom's output contribution before the finite difference is measured, whereas $\eta$ amplifies signal and noise after measurement; no scalar learning-rate multiplier can undo the $1/r^2$ active FD-SNR decay diagnosed by Theorem~\ref{thm:directional_collapse}.

\paragraph{Systems profile.}
The profiling panel checks that the gains are not purchased with extra systems cost. AR1-ZO matches MeZO+LoRA in step time and peak VRAM ($273$\,ms / $13.9$\,GB), while LOZO incurs about $15\%$ slower steps and $23\%$ higher peak VRAM. Thus the alternating rank-$1$ schedule and topology-aware rescaling preserve the pure black-box two-forward-pass budget.

\section{Theoretical Proofs}
\label{sec:appendix_proof}

\subsection{Preliminaries and Assumptions}
\label{sec:appendix_preliminaries}

Before presenting the proofs, we formalize the necessary mathematical
assumptions and standard properties of the zeroth-order oracle.

\begin{assumption}[$L$-smoothness]
\label{assum:smoothness}
The objective function $\mathcal{L}(\theta)$ is differentiable and its
gradient is $L$-Lipschitz. For any $\theta, \theta'$, there exists a
constant $L > 0$ such that:
\begin{equation}
\|\nabla \mathcal{L}(\theta') - \nabla \mathcal{L}(\theta)\|\le L\|\theta' - \theta\|.
\end{equation}
This yields the standard quadratic upper bound:
\begin{equation}
\mathcal{L}(\theta') \le \mathcal{L}(\theta)+\langle \nabla \mathcal{L}(\theta),\, \theta' - \theta \rangle+\frac{L}{2} \|\theta' - \theta\|^2.
\end{equation}
\end{assumption}

\begin{assumption}[Bounded Evaluation Noise]
\label{assum:noise}
In practice, we only have access to a stochastic realization of the
loss function.  Paired evaluations may share the same mini-batch or
other common random numbers; after conditioning on this shared
information, we model the remaining oracle error as
$\tilde{\mathcal{L}}(\theta) = \mathcal{L}(\theta) + \xi$,
where $\xi$ is an independent zero-mean residual satisfying:
\begin{equation}
\E\!\left[\xi\right]=0,\qquad \E\!\left[\xi^2\right]=\sigma_\xi^2.
\end{equation}
The residual errors in two evaluations
$\tilde{\mathcal{L}}(\theta + \mu \vz)$ and
$\tilde{\mathcal{L}}(\theta - \mu \vz)$ are modeled as independent
realizations $\xi_+$ and $\xi_-$.
\end{assumption}

\begin{assumption}[Active-atom gradient regularity]
\label{assum:unitnorm}
Along the analyzed iterates, the unscaled active factor-coordinate
gradient
\[
\vg_t
\;\defeq\;
\begin{bmatrix}\mG\,\va_{k(t)}\\[2pt]\mG^\top\vb_{k(t)}\end{bmatrix}
\in\R^{q},
\qquad q=\dout+\din,
\]
satisfies $\|\vg_t\|=\Theta(1)$ in the LoRA rank $r$; that is,
$\|\vg_t\|$ admits upper and lower bounds that depend only on the
dense gradient $\mG$ and on the smoothness/Lipschitz constants of
$\mathcal L$, not on $r$.  We further assume the queried atom has
bounded factor norms,
$\|\vb_{k(t)}\|,\|\va_{k(t)}\|=\Theta(1)$, the standard local
regularity used to interpret the active alignment $\beta$ in
Theorem~\ref{thm:structural_advantage}.  Both conditions are
\emph{local} statements about the queried atom $k(t)$ and do not
constrain the magnitudes of the remaining $r-1$ stored atoms. The
upper-bound half $\|\vg_t\|=\mathcal O(1)$ alone suffices for the
convergence bound in Theorem~\ref{thm:signal_restoration} Part~II; the
full $\Theta(1)$ regularity is used by Part~I to make
$\mathsf{SNR}_t$ rank-invariant in a precise rate sense.
\end{assumption}

\begin{assumption}[Stable Represented Adapter Scale]
\label{assum:stable_adapter}
The represented adapter remains in a bounded forward-scale regime
along the analyzed iterates:
\[
\left\|\frac{\gamma}{r}\sum_{k=1}^{r}\vb_k\va_k^\top\right\|_F=\mathcal O(1).
\]
For AR1-ZO, where $\gamma/r=\alpha$, this is equivalent to
$\bigl\|\sum_{k=1}^{r}\vb_k\va_k^\top\bigr\|_F=\mathcal O(1)$.  The
assumption rules out interpreting the scaling law as an uncontrolled
increase of the entire adapter output; the active FD-SNR algebra
itself only depends on the active coefficient $\gamma/r$.

\paragraph*{Relation to Assumption~\ref{assum:unitnorm}.}
Assumption~\ref{assum:stable_adapter} concerns the represented adapter
$\|\Delta\mW\|_F$, a sum over all $r$ stored atoms, while
Assumption~\ref{assum:unitnorm} concerns the queried atom only (its
active gradient norm and its factor norms).  The two statements
constrain different physical objects---one global summed, one local
active---and are used in disjoint segments of the proofs.  In
particular, with per-atom factor norms $\Theta(1)$ on the queried
atom, generic incoherent configurations of the remaining stored atoms
would give $\bigl\|\sum_k\vb_k\va_k^\top\bigr\|_F=\Theta(\sqrt r)$, so
Assumption~\ref{assum:stable_adapter} is \emph{not} an incoherence
condition: it is a spectrum-control condition on the trained adapter,
requiring $\mB\mA$ to remain in a regime where a small number of
singular components dominate and the Frobenius norm of the sum stays
$\mathcal O(1)$ in $r$.  This regime is observed empirically in
trained LoRA, where the effective rank of $\mB\mA$ is well below the
nominal $r$; it is a property of the optimization trajectory rather
than a geometric prior on the atoms, and it imposes no constraint on
the active-atom regularity asserted by
Assumption~\ref{assum:unitnorm}.
\end{assumption}

\subsection{Auxiliary Lemmas}
\label{sec:appendix_lemmas}

\begin{lemma}[Active-Atom Gradient under Factored Parameterization]
\label{lemma:active_grad}
Consider the factored update
$\Delta \mW = \frac{\gamma}{r}\sum_{k=1}^{r}
  \vb_k\va_k^\top$,
where $\mG = \nabla_{\mW} \mathcal{L} \in \R^{\dout
  \times \din}$
is the dense matrix gradient.  The active-atom gradient with
respect to atom $k$ is, writing $\vb$ and $\va$ in gradient
subscripts for the active factors $\vb_k$ and $\va_k$:
\begin{equation}
\nabla_{\vb}\mathcal{L}=\frac{\gamma}{r}\,\mG\va_k,\qquad \nabla_{\va}\mathcal{L}=\frac{\gamma}{r}\,\mG^\top\vb_k.
\end{equation}
\end{lemma}

\begin{proof}
By the chain rule, since
$\mW = \mW_0 + \frac{\gamma}{r}\sum_j
  \vb_j\va_j^\top$:
\[
\frac{\partial \mW}{\partial [\vb]_i}=\frac{\gamma}{r}\,\ve_i\va_k^\top.
\]
Therefore:
\[
[\nabla_{\vb}\mathcal{L}]_i=\operatorname{tr}\!\left(\mG^\top \cdot \frac{\gamma}{r}\,\ve_i\va_k^\top\right)=\frac{\gamma}{r}\,\ve_i^\top \mG\va_k=\frac{\gamma}{r}\,[\mG\va_k]_i.
\]
Stacking over all $i$ yields
$\nabla_{\vb}\mathcal{L}
  = \frac{\gamma}{r}\,\mG\va_k$.
The derivation for $\nabla_{\va}\mathcal{L}$ is analogous.
\end{proof}

\begin{lemma}[Projected ZO-RGE moment bounds]
\label{lemma:zo_properties}
Let $d$ denote the full LoRA coordinate dimension for the adapted
matrix, and let $\mP \in \R^{d \times q}$ be the
orthogonal basis spanning atom $k$'s active coordinates,
with $\mP^\top \mP = \mI_{q}$ and
$q = \din + \dout$, and let
$\vz \sim \mathcal{N}(0, \mI_{q})$.  The two-point
Randomized Gradient Estimator (RGE):
\[
\hat{\vg}=\frac{\tilde{\mathcal{L}}(\theta+\mu\mP\vz)-\tilde{\mathcal{L}}(\theta-\mu\mP\vz)}{2\mu}\vz
\]
satisfies the following standard projected two-point moment bounds
under the smoothness conditions needed for the central finite-difference
remainder:
\begin{enumerate}
\item \textbf{Bias Bound:}
  $\|\E_{\vz,\xi}\!\left[\hat{\vg}\right]
    - \nabla_k\mathcal{L}\|
    = \mathcal{O}(\mu^2 L\, q)$.
\item \textbf{Second Moment Bound:}
  $\E_{\vz,\xi}\!\left[\|\hat{\vg}\|^2\right]
    \le (q+3)\,
        \|\nabla_k\mathcal{L}\|^2
      + C_1 \mu^2 L^2 q^3
      + \frac{C_2\, q\,
            \sigma_\xi^2}{2\mu^2}$,\quad
  for universal constants $C_1, C_2 > 0$.
  The leading coefficient $(q+3)$ is a clean upper
  bound that absorbs signal--bias cross terms arising from
  the $\vz$-dependent Hessian remainder; the exact signal-only
  second moment is $(q+2)\|\nabla_k
  \mathcal{L}\|^2$.
\end{enumerate}
\end{lemma}

\textit{Proof of Lemma~\ref{lemma:zo_properties} follows from standard
zeroth-order optimization literature
\cite{nesterov2017random, liu2018zeroth}.}

\subsection{Proof of Lemma~\ref{lem:atom_minimality}: Complete Rank-\texorpdfstring{$1$}{1} Atom Blocks}

\begin{proof}
For the atom map $\phi_k(\vb_k,\va_k)=\vb_k\va_k^\top$, the
Fr\'echet derivative in direction $(\delta\vb,\delta\va)$ is
\[
D\phi_k(\delta\vb,\delta\va)=\left.\frac{d}{d\epsilon}(\vb_k+\epsilon\delta\vb)(\va_k+\epsilon\delta\va)^\top\right|_{\epsilon=0}=\delta\vb\,\va_k^\top+\vb_k\,\delta\va^\top .
\]
If only $\vb_k$ is perturbed, the attainable tangent matrices have the
form $\delta\vb\,\va_k^\top$ and keep the right factor fixed at
$\va_k$.  If only $\va_k$ is perturbed, they have the form
$\vb_k\,\delta\va^\top$ and keep the left factor fixed at $\vb_k$.
When both factors are nonzero, each one-sided family is a strict
subset of the full factor-coordinate tangent generated by
$D\phi_k(\delta\vb,\delta\va)$.  Therefore a complete atom-local
LoRA-native factor-coordinate perturbation includes the matched pair
$(\vb_k,\va_k)$.

The stored coordinates of this matched pair have dimension
$\dout+\din$.  The rank-$1$ matrix manifold itself has one scaling
gauge redundancy, since
$(\vb,\va)$ and $(c\vb,\va/c)$ represent the same matrix for
$c\neq0$, so its intrinsic dimension is $\dout+\din-1$.  AR1-ZO,
however, queries the stored LoRA factors through a black-box oracle,
not a quotient manifold coordinate chart; the factor-coordinate query
therefore perturbs $\dout+\din$ variables.  A complete block of $m$
matched atoms is the direct product of $m$ such blocks and has query
dimension $m(\dout+\din)$.
\end{proof}

\subsection{Proof of Proposition~\ref{prop:active_atom_dimension}: Active-Atom ZO Dimension}

\begin{proof}
Fix iteration $t$ and write $k=k(t)$. Let $\mP_t$ inject an active
perturbation $\vz\in\R^{q}$ into the
factor coordinates of atom $k$, and define the proof-local
active-block Gaussian smoothing
$\mathcal L_{\mu,k(t)}(\theta)=
\E_{\vz}\!\left[
\mathcal L(\theta+\mu\mP_t\vz)\right]$.
The standard Gaussian smoothing identity for the two-point randomized
gradient estimator gives
\[
\E_{\vz,\xi}\!\left[\frac{\tilde{\mathcal L}(\theta+\mu\mP_t\vz)-\tilde{\mathcal L}(\theta-\mu\mP_t\vz)}{2\mu}\vz\right]=\nabla_{k(t)}\mathcal L_{\mu,k(t)}(\theta),
\]
because the residual oracle noise is zero mean after conditioning on
the paired query protocol.

Lemma~\ref{lemma:zo_properties} applies with active dimension
$q=\dout+\din$.  Combining its second-moment and bias bounds gives
the usual mean-squared error form
\[
\E\!\left[\|\hgrad_{k(t)}\mathcal L-\nabla_{k(t)}\mathcal L(\theta)\|^2\right]\lesssim q\|\nabla_{k(t)}\mathcal L(\theta)\|^2+L^2\mu^2q^3+\frac{\sigma_\xi^2q}{\mu^2},
\]
up to universal constants.  The dimension appearing in every term is
the dimension of the sampled vector $\vz$,
namely $q$.  If instead
one perturbs all $r$ LoRA atoms of the same matrix at once, the same
standard ZO bound is obtained with $q$ replaced by $rq$.
\end{proof}

\subsection{Proof of Theorem~\ref{thm:directional_collapse}: Directional Collapse under Naive Scaling}

\textbf{Theorem~\ref{thm:directional_collapse} (Restated).}
\textit{Under Assumptions~\ref{assum:smoothness},
\ref{assum:noise}, \ref{assum:unitnorm} and naive $\gamma = \alpha$,
the active signal is $\Theta(r^{-1})$ and active FD-SNR
$\mathsf{SNR}_t=\Theta(r^{-2})$.  With
$r_c=\Theta(\alpha\mu\|\vg_t\|/\sigma_\xi)=\Theta(\alpha\mu/\sigma_\xi)$,
$r\gg r_c$ makes the expected cosine with the active gradient vanish.}

\begin{proof}
Let $\theta$ denote the full parameter vector.  At step $t$, we
freeze all parameters except atom $k$, the coordinate pair
$(\vb_k,\va_k)$ with dimension $q = \din + \dout$.

\textbf{Step 1: Signal Magnitude Decay.}

Under the naive scaling $\gamma = \alpha$,
Lemma~\ref{lemma:active_grad} gives:
\begin{equation}
\nabla_{\vb}\mathcal{L}=\frac{\alpha}{r}\,\mG\va_k,\qquad \nabla_{\va}\mathcal{L}=\frac{\alpha}{r}\,\mG^\top\vb_k.
\end{equation}
Define the unscaled active factor-coordinate gradient:
\begin{equation}
\vg_k \;=\;
  \begin{bmatrix} \mG\va_k \\\ \mG^\top\vb_k \end{bmatrix}
  \;\in\; \R^{q}.
\end{equation}
Then at rank $r$:
\begin{equation}
\nabla_k\mathcal{L}=\frac{\alpha}{r}\,\vg_k,\qquad \|\nabla_k\mathcal{L}\|=\frac{\alpha}{r}\,\|\vg_k\|.
\end{equation}
By Assumption~\ref{assum:unitnorm}, $\|\vg_k\|=\Theta(1)$ in $r$, so
the signal magnitude decays as $\Theta(1/r)$ rather than merely as
$\mathcal O(1/r)\cdot\|\vg_k\|$ with a possibly $r$-dependent
prefactor. \hfill$\square$

\textbf{Step 2: Finite-Difference Signal--Noise Decomposition.}

The two-point FD observation along
$\vz \sim \mathcal{N}(0, \mI_{q})$ is decomposed via
Taylor expansion:
\begin{equation}
\label{eq:fd_decomp}
\Delta_\mu
  \;=\;
  \frac{\tilde{\mathcal{L}}(\theta+\mu \mP\vz)
    - \tilde{\mathcal{L}}(\theta-\mu \mP\vz)}{2\mu}
  \;=\;
  \underbrace{\vz^\top \nabla_k\mathcal{L}}
    _{\text{\normalfont signal}}
  + \underbrace{\mathcal{O}(\mu^2)}_{\text{\normalfont bias}\footnotemark}
  + \underbrace{\frac{\xi_+ - \xi_-}{2\mu}}
    _{\text{\normalfont noise }\zeta}.
\end{equation}
\footnotetext{The $\mathcal{O}(\mu^2)$ bias of central
differences assumes standard higher-order smoothness (bounded
third derivatives).  Under only $L$-smoothness
(Assumption~\ref{assum:smoothness}), the bias is
$\mathcal{O}(\mu)$, which is still strictly dominated by the
stochastic noise $\zeta$ for small $\mu$.}
The full ZO estimator is then
$\hat{\vg} = \Delta_\mu \cdot \vz
  = (\vz^\top \nabla_k\mathcal{L})\,\vz
    + \mathcal{O}(\mu^2)\,\vz
    + \zeta\,\vz$.

\textit{Signal power:}
\begin{equation}
\E\!\left[(\vz^\top\nabla_k\mathcal{L})^2\right]=\|\nabla_k\mathcal{L}\|^2=\frac{\alpha^2\|\vg_k\|^2}{r^2}.
\end{equation}

\textit{Residual noise power} (independent of $r$):
\begin{equation}
\E\!\left[\zeta^2\right]=\frac{\sigma_\xi^2}{2\mu^2}.
\end{equation}
\hfill$\square$

\textbf{Step 3: Active FD-SNR and Critical Rank.}

We define the active FD-SNR as the ratio of
signal power to residual noise power:
\begin{equation}
\mathsf{SNR}(r)=\frac{\alpha^2\|\vg_k\|^2/r^2}{\sigma_\xi^2/(2\mu^2)}=\frac{2\alpha^2\mu^2\|\vg_k\|^2}{r^2\sigma_\xi^2}.
\end{equation}
This decays as $\mathcal{O}(1/r^2)$.  Setting
$\mathsf{SNR} = 1$ and solving for $r$:
\begin{equation}
r_c=\frac{\sqrt{2}\,\alpha\mu\|\vg_k\|}{\sigma_\xi}=\Theta\!\left(\frac{\alpha\mu\|\vg_k\|}{\sigma_\xi}\right).
\end{equation}
Under Assumption~\ref{assum:unitnorm}, $\|\vg_k\|=\Theta(1)$, so
$r_c=\Theta(\alpha\mu/\sigma_\xi)$ depends only on the smoothing
radius, the active LoRA scale, and the residual evaluation noise.
\hfill$\square$

\textbf{Step 4: Cosine Similarity Collapse.}

From the decomposition in~\eqref{eq:fd_decomp}, the ZO estimator
is $\hat{\vg} = (\vz^\top \vg)\,\vz + \zeta\,\vz$
(dropping the $\mathcal{O}(\mu^2)$ bias under the standard
assumption that the smoothing parameter $\mu$ is chosen
sufficiently small, e.g.,
$\mu = \mathcal{O}(q^{-1/2})$, such that the
finite-difference bias is strictly dominated by the stochastic
noise $\zeta$),
where $\vg = \nabla_k\mathcal{L}$.

We analyze $\E\!\left[\cos(\hat{\vg},\, \vg)\right]$ by
computing the numerator and denominator separately.

\textit{Numerator:}
\begin{equation}
\E\!\left[\hat{\vg}^\top \vg\right]=\E\!\left[(\vz^\top \vg)^2\right]+\E\!\left[\zeta\,(\vz^\top \vg)\right]=\|\vg\|^2.
\end{equation}
Here, $\E\!\left[\zeta\,(\vz^\top \vg)\right] = 0$ since
$\zeta$ and $\vz$ are independent.

\textit{Denominator:}
Since the bias has been dropped, the second moment of the
bias-free estimator $\hat{\vg} = (\vz^\top \vg)\vz
+ \zeta \vz$ uses the exact signal coefficient
$(q+2)$, and the smoothness remainder
$C_1 \mu^2 L^2 q^3$ from
Lemma~\ref{lemma:zo_properties} is likewise negligible under
$\mu = \mathcal{O}(q^{-1/2})$:
\begin{equation}
\E\!\left[\|\hat{\vg}\|^2\right]=(q+2)\|\vg\|^2+\frac{\sigma_\xi^2q}{2\mu^2}.
\end{equation}

By the concentration of measure in high dimensions and standard
zeroth-order bounds (cf.~\cite{nesterov2017random}, Lemma~3),
the random variable $\|\hat{\vg}\|^2$ tightly concentrates around
its expectation. Thus, the expected cosine similarity can be
approximated as:
\begin{equation}
\begin{aligned}
\E\!\left[\cos(\hat{\vg},\, \vg)\right]
  &\;\asymp\;
  \frac{\|\vg\|^2}
       {\|\vg\|\,
        \sqrt{(q+2)\,\|\vg\|^2
          + \sigma_\xi^2\,q/(2\mu^2)}} \\
  &\;=\;
  \frac{\|\vg\|}
       {\sqrt{(q+2)\,\|\vg\|^2
          + \sigma_\xi^2\,q/(2\mu^2)}}.
\end{aligned}
\end{equation}

Substituting $\|\vg\| = \alpha\|\vg_k\|/r$:
\begin{equation}
\E\!\left[\cos(\hat{\vg}_{\text{\normalfont naive}},\,
  \vg)\right]
  \;\asymp\;
  \frac{\alpha\|\vg_k\|/r}
       {\sqrt{(q+2)\,\alpha^2\|\vg_k\|^2/r^2
          + \sigma_\xi^2\,q/(2\mu^2)}}.
\end{equation}

When $r \gg r_c$, the residual noise term dominates the denominator:
\begin{equation}
\frac{\sigma_\xi^2\,q}{2\mu^2}
  \;\gg\;
  \frac{(q+2)\,\alpha^2\|\vg_k\|^2}{r^2},
\end{equation}
so the denominator
$\approx \sqrt{\sigma_\xi^2\,q/(2\mu^2)}$,
and:
\begin{equation}
\E\!\left[\cos(\hat{\vg}_{\text{\normalfont naive}},\,
  \vg)\right]
  \;\approx\;
  \frac{\alpha\|\vg_k\|/r}
       {\sigma_\xi\,
        \sqrt{q/(2\mu^2)}}
  \;=\;
  \frac{\sqrt{2}\,\alpha\mu\,\|\vg_k\|}
       {r\,\sigma_\xi\,
        \sqrt{q}}
  \;=\;
  \mathcal{O}\!\left(
    \frac{r_c}{r\,\sqrt{q}}
  \right)
  \;\to\; 0.
\end{equation}

This establishes \textbf{directional collapse}: the ZO estimate
loses all correlation with the true gradient direction.  The
inverse dependence on both $r$ and $\sqrt{q}$
implies that in high-dimensional embedding spaces, even moderate
ranks can trigger collapse when $\sigma_\xi$ is
non-vanishing.
\end{proof}

\begin{remark}
This collapse is qualitatively different from slow convergence.
In first-order optimization, scaling the gradient by $1/r$ merely
reduces step size while preserving direction.  In ZO, once the
signal falls below the noise floor, the \emph{direction itself}
is destroyed.  No learning rate adjustment can repair a
noise-dominated direction estimate.
\end{remark}

\subsection{Proof of Theorem~\ref{thm:signal_restoration}: Signal Restoration and Active-Atom Stationarity}

\textbf{Theorem~\ref{thm:signal_restoration} (Restated).}
\textit{
\textbf{(Part I, FD-SNR.)} Under
Assumptions~\ref{assum:smoothness}, \ref{assum:noise},
\ref{assum:unitnorm} and topology-aware scaling $\gamma=\alpha r$, the
active FD-SNR satisfies
\[
\mathsf{SNR}_t\asymp\frac{\alpha^2\mu^2\|\vg_t\|^2}{\sigma_\xi^2}=\Theta(1)\;\text{ in }r.
\]
\textbf{(Part II, Convergence.)} Under
Assumptions~\ref{assum:smoothness}, \ref{assum:noise},
\ref{assum:stable_adapter}, the upper-bound half of
Assumption~\ref{assum:unitnorm} (i.e., $\|\vg_t\|=\mathcal O(1)$ along
iterates, with bounded active factor norms
$\|\vb_{k(t)}\|,\|\va_{k(t)}\|=\mathcal O(1)$), and the projected RGE
moment bounds of Lemma~\ref{lemma:zo_properties}, with topology-aware
scaling $\gamma=\alpha r$,}
\begin{equation}
\frac{1}{T}\sum_{t=0}^{T-1}
  \E\!\left[\|\nabla_{k(t)}
  \mathcal{L}(\theta_t)\|^2\right]
  \leq \mathcal{O}\!\left(\frac{1}{\sqrt{T}}\right)
  + \mathcal{O}\!\left(
      \mu^4 L^2 q^2
      + \frac{\sigma_\xi^2\,
            q}{T^{1/2}\mu^2}
    \right).
\end{equation}
\textit{Part I uses Assumption~\ref{assum:unitnorm} to make the active
FD-SNR a true rank-invariant quantity; Part II uses
Assumption~\ref{assum:stable_adapter} to rule out trivial inflation of
the represented adapter under $\gamma=\alpha r$, and depends on
Assumption~\ref{assum:unitnorm} only through the upper bound
$\|\vg_t\|=\mathcal O(1)$.}

\begin{proof}

\noindent\textbf{Part I: Signal Restoration}

\medskip
\textbf{Step 1: Active-atom gradient under
$\gamma = \alpha r$.}

By Lemma~\ref{lemma:active_grad}, the active-atom gradient
with respect to atom $k$ under general scaling $\gamma$ is:
\begin{equation}
\nabla_{\vb}\mathcal{L}=\frac{\gamma}{r}\,\mG\va_k,\qquad \nabla_{\va}\mathcal{L}=\frac{\gamma}{r}\,\mG^\top\vb_k.
\end{equation}
Substituting $\gamma = \alpha r$:
\begin{equation}
\nabla_{\vb}\mathcal{L}=\frac{\alpha r}{r}\,\mG\va_k=\alpha\,\mG\va_k,\qquad \nabla_{\va}\mathcal{L}=\alpha\,\mG^\top\vb_k.
\end{equation}
The factor $r$ in $\gamma$ exactly cancels the $1/r$ from
the factored parameterization.  The concatenated active-atom
gradient is therefore:
\begin{equation}
\nabla_k\mathcal{L}
  = \alpha\,
    \begin{bmatrix}
      \mG\va_k \\\ \mG^\top\vb_k
    \end{bmatrix}
  = \alpha\vg_k,
\end{equation}
In particular,
$\|\nabla_k\mathcal{L}\|
  = \alpha\|\vg_k\|$,
which is independent of the total rank $r$.
\hfill$\square$

\medskip
\textbf{Step 2: Active FD-SNR independence of $r$.}

From Step~2 of the naive-scaling proof above, the active FD-SNR is:
\begin{equation}
\mathsf{SNR}=\frac{\text{\normalfont signal power}}{\text{\normalfont noise power}}=\frac{\|\nabla_k\mathcal{L}\|^2}{\sigma_\xi^2/(2\mu^2)}=\frac{2\alpha^2\mu^2\|\vg_k\|^2}{\sigma_\xi^2},
\end{equation}
which is constant in $r$.  By Assumption~\ref{assum:unitnorm},
$\|\vg_t\|=\Theta(1)$, hence $\mathsf{SNR}_t=\Theta(1)$ in $r$.
Signal restoration is \emph{measurement-level} and is not contingent on
the represented adapter scale. \hfill$\square$


\bigskip
\noindent\textbf{Part II: Convergence}

\medskip
\noindent
The convergence bound below uses
Assumption~\ref{assum:stable_adapter} to rule out the trivial reading
that $\gamma=\alpha r$ merely inflates the represented adapter; it does
\emph{not} require the lower-bound half of
Assumption~\ref{assum:unitnorm}, since the descent inequality and the
projected RGE moment bounds depend only on the active dimension $q$,
on Lipschitz/smoothness constants, and on the upper bound
$\|\vg_t\|=\mathcal O(1)$.

\medskip
\textbf{Step 3: Per-step descent (Descent Lemma).}

At step $t$, the active atom is $k(t)$ and the update is
$\theta_{t+1} = \theta_t - \eta_t\,\hat{\vg}_t$, where
$\hat{\vg}_t$ is the projected ZO estimator restricted to
the active atom $k(t)$ as defined in
Lemma~\ref{lemma:zo_properties}.

By Assumption~\ref{assum:smoothness} ($L$-smoothness),
the standard descent inequality applied to the full
gradient gives
$\mathcal{L}(\theta_{t+1})
  \leq \mathcal{L}(\theta_t)
  - \eta_t\,\langle
      \nabla\mathcal{L}(\theta_t),\,
      \hat{\vg}_t
    \rangle
  + \frac{\eta_t^2 L}{2}\,\|\hat{\vg}_t\|^2$.
Since $\hat{\vg}_t$ is supported entirely on the
active atom $k(t)$ (all other atom coordinates are zero), the inner product
reduces to
$\langle \nabla\mathcal{L}(\theta_t),\,\hat{\vg}_t
  \rangle
  = \langle \nabla_{k(t)}
      \mathcal{L}(\theta_t),\,\hat{\vg}_t \rangle$.
Thus:
\begin{equation}
\label{eq:descent_raw}
\mathcal{L}(\theta_{t+1})
  \leq \mathcal{L}(\theta_t)
  - \eta_t\,\langle
      \nabla_{k(t)}\mathcal{L}(\theta_t),\,
      \hat{\vg}_t
    \rangle
  + \frac{\eta_t^2 L}{2}\,\|\hat{\vg}_t\|^2.
\end{equation}
Taking expectations over $\vz_t$ and $\xi_t$ conditioned on
$\theta_t$:
\begin{equation}
\label{eq:descent_expect}
\E\!\left[\mathcal{L}(\theta_{t+1})\right]
  \leq \mathcal{L}(\theta_t)
  - \eta_t\,\langle
      \nabla_{k(t)}\mathcal{L},\,
      \E\!\left[\hat{\vg}_t\right]
    \rangle
  + \frac{\eta_t^2 L}{2}\,
    \E\!\left[\|\hat{\vg}_t\|^2\right].
\end{equation}

From the bias bound in Lemma~\ref{lemma:zo_properties}:
$\E\!\left[\hat{\vg}_t\right]
  = \nabla_{k(t)}\mathcal{L}
    + b_t$,
where $\|b_t\| = \mathcal{O}(\mu^2 L\,q)$.
Therefore:
\begin{align}
\langle
  \nabla_{k(t)}\mathcal{L},\,
  \E\!\left[\hat{\vg}_t\right]
\rangle
&= \|\nabla_{k(t)}\mathcal{L}\|^2
   + \langle
       \nabla_{k(t)}\mathcal{L},\,
       b_t
     \rangle \nonumber\\
&\geq \|\nabla_{k(t)}\mathcal{L}\|^2
   - \|\nabla_{k(t)}\mathcal{L}\|\,
     \|b_t\|
   \qquad \text{(Cauchy--Schwarz)},
\end{align}
where $\|b_t\| \leq C_b\,\mu^2 L\,q$
for an absolute constant $C_b > 0$.

Substituting into~\eqref{eq:descent_expect}:
\begin{equation}
\label{eq:descent_full}
\E\!\left[\mathcal{L}(\theta_{t+1})\right]
  \leq \mathcal{L}(\theta_t)
  - \eta_t\,
    \|\nabla_{k(t)}\mathcal{L}\|^2
  + \eta_t\,C_b\,\mu^2 L\,q\,
    \|\nabla_{k(t)}\mathcal{L}\|
  + \frac{\eta_t^2 L}{2}\,
    \E\!\left[\|\hat{\vg}_t\|^2\right].
\end{equation}

\medskip
\textbf{Step 4: Bounding
$\E\!\left[\|\hat{\vg}_t\|^2\right]$.}

From the second moment bound in
Lemma~\ref{lemma:zo_properties}:
\begin{equation}
\label{eq:second_moment}
\E\!\left[\|\hat{\vg}_t\|^2\right]
  \leq (q+3)\,
       \|\nabla_{k(t)}\mathcal{L}\|^2
     + C_1\,\mu^2 L^2\,q^3
     + \frac{C_2\,\sigma_\xi^2\,
           q}{2\mu^2}.
\end{equation}
(The leading coefficient $(q+3)$ is the Gaussian
fourth-moment coefficient $(q+2)$ for the pure signal
term, plus an additional $+1$ that absorbs signal--bias
cross terms from the $\vz$-dependent Hessian remainder; see
Lemma~\ref{lemma:zo_properties}.)

Under $\gamma = \alpha r$, we have
$\|\nabla_{k(t)}\mathcal{L}\|^2
  = \alpha^2\|\vg_t\|^2$
by Part~I, which does \textbf{not grow with} $r$.
Therefore, all terms in the bound~\eqref{eq:second_moment}
are independent of $r$.  Part~II's descent estimate uses only the
upper-bound half $\|\vg_t\|=\mathcal O(1)$ of
Assumption~\ref{assum:unitnorm}, not the $\Theta(1)$ lower bound that
Part~I needs to make $\mathsf{SNR}_t=\Theta(1)$ a precise rank rate.
Note that Assumption~\ref{assum:stable_adapter} alone does \emph{not}
imply $\|\vg_t\|=\mathcal O(1)$: the represented-adapter bound
$\|(\gamma/r)\sum_k\vb_k\va_k^\top\|_F=\mathcal O(1)$ controls a sum
over all atoms with gauge freedom $(\vb_k,\va_k)\mapsto(c\vb_k,\va_k/c)$,
while $\|\vg_t\|^2\le\|\mG\|_2^2(\|\va_{k(t)}\|^2+\|\vb_{k(t)}\|^2)$
also depends on the dense gradient $\mG$ and on the active atom's
individual factor norms. The two assumptions are therefore
complementary and both used in Part~II.

\medskip
\textbf{Step 5: Absorbing the cross-term via
Young's inequality.}

The cross-term in~\eqref{eq:descent_full} is handled by
Young's inequality ($ab \leq \frac{a^2}{4} + b^2$):
\begin{equation}
\eta_t\,C_b\,\mu^2 L\,q\,
  \|\nabla_{k(t)}\mathcal{L}\|
  \leq
  \frac{\eta_t}{4}\,
    \|\nabla_{k(t)}\mathcal{L}\|^2
  + \eta_t\,C_b^2\,\mu^4 L^2\,q^2.
\end{equation}

Substituting into~\eqref{eq:descent_full}:
\begin{equation}
\label{eq:descent_absorbed}
\E\!\left[\mathcal{L}(\theta_{t+1})\right]
  \leq \mathcal{L}(\theta_t)
  - \frac{3\eta_t}{4}\,
    \|\nabla_{k(t)}\mathcal{L}\|^2
  + \eta_t\,C_b^2\,\mu^4 L^2\,q^2
  + \frac{\eta_t^2 L}{2}\,
    \E\!\left[\|\hat{\vg}_t\|^2\right].
\end{equation}

\medskip
\textbf{Step 6: Learning rate choice.}

Choose a constant learning rate
$\eta_t = \eta = \frac{c}{\sqrt{T}}$
for a constant $c > 0$.  For sufficiently large $T$
(specifically, $T \geq c^2 \cdot 4L^2(q+3)^2$),
this satisfies:
\begin{equation}
\frac{\eta L}{2}(q+3)\leq\frac{1}{4},\qquad\text{i.e.,}\qquad \eta\leq\frac{1}{2L(q+3)}.
\end{equation}
This ensures that the
$\frac{\eta^2 L}{2}(q+3)
  \|\nabla_{k(t)}\mathcal{L}\|^2$
term from~\eqref{eq:second_moment} can be absorbed
into the leading descent term.

Specifically, substituting~\eqref{eq:second_moment}
into~\eqref{eq:descent_absorbed}:
\begin{align}
\E\!\left[\mathcal{L}(\theta_{t+1})\right]
  &\leq \mathcal{L}(\theta_t)
  - \frac{3\eta}{4}\,
    \|\nabla_{k(t)}\mathcal{L}\|^2
  + \frac{\eta^2 L}{2}(q+3)\,
    \|\nabla_{k(t)}\mathcal{L}\|^2
  \nonumber\\
  &\quad
  + \eta\,C_b^2\,\mu^4 L^2\,q^2
  + \frac{\eta^2 L}{2}\!\left(
      C_1\,\mu^2 L^2\,q^3
      + \frac{C_2\,\sigma_\xi^2\,
            q}{2\mu^2}
    \right).
\end{align}
Applying the learning rate constraint
$\frac{\eta L}{2}(q+3) \leq \frac{1}{4}$:
\begin{equation}
-\frac{3\eta}{4}\|\nabla_{k(t)}\mathcal{L}\|^2+\frac{\eta}{4}\|\nabla_{k(t)}\mathcal{L}\|^2=-\frac{\eta}{2}\|\nabla_{k(t)}\mathcal{L}\|^2.
\end{equation}
This simplifies to:
\begin{equation}
\label{eq:descent_clean}
\E\!\left[\mathcal{L}(\theta_{t+1})\right]
  \leq \mathcal{L}(\theta_t)
  - \frac{\eta}{2}\,
    \|\nabla_{k(t)}\mathcal{L}\|^2
  + \eta\,\Phi,
\end{equation}
where $\Phi$ collects all terms independent of
$\|\nabla_{k(t)}\mathcal{L}\|^2$:
\begin{equation}
\label{eq:Phi_def}
\Phi \;=\;
  C_b^2\,\mu^4 L^2\,q^2
  + \frac{\eta L}{2}\!\left(
      C_1\,\mu^2 L^2\,q^3
      + \frac{C_2\,\sigma_\xi^2\,
            q}{2\mu^2}
    \right).
\end{equation}

\medskip
\textbf{Step 7: Telescoping sum.}

Rearranging~\eqref{eq:descent_clean} and summing from
$t = 0$ to $T-1$:
\begin{equation}
\frac{\eta}{2}\sum_{t=0}^{T-1}
  \E\!\left[
    \|\nabla_{k(t)}\mathcal{L}(\theta_t)\|^2
  \right]
  \leq
  \mathcal{L}(\theta_0) - \mathcal{L}^*
  + T\eta\,\Phi.
\end{equation}
Dividing both sides by $T\eta/2$:
\begin{equation}
\label{eq:avg_grad}
\frac{1}{T}\sum_{t=0}^{T-1}
  \E\!\left[
    \|\nabla_{k(t)}\mathcal{L}(\theta_t)\|^2
  \right]
  \leq
  \frac{2(\mathcal{L}(\theta_0) - \mathcal{L}^*)}
       {T\eta}
  + 2\Phi.
\end{equation}

\medskip
\textbf{Step 8: Final bound.}

Substituting $\eta = c/\sqrt{T}$
into~\eqref{eq:avg_grad}:
\begin{equation}
\frac{1}{T}\sum_{t=0}^{T-1}
  \E\!\left[
    \|\nabla_{k(t)}\mathcal{L}(\theta_t)\|^2
  \right]
  \leq
  \underbrace{
    \frac{2(\mathcal{L}(\theta_0) - \mathcal{L}^*)}
         {c\sqrt{T}}
  }_{\mathcal{O}(1/\sqrt{T})}
  + 2\Phi.
\end{equation}

Expanding $\Phi$ from~\eqref{eq:Phi_def} with
$\eta = c/\sqrt{T}$:
\begin{align}
2\Phi
  &= 2C_b^2\,\mu^4 L^2\,q^2
     + \frac{c L}{\sqrt{T}}\!\left(
         C_1\,\mu^2 L^2\,q^3
         + \frac{C_2\,\sigma_\xi^2\,
               q}{2\mu^2}
       \right) \nonumber\\
  &= \underbrace{
       \mathcal{O}(\mu^4 L^2\,q^2)
     }_{\text{\normalfont smoothing bias (persistent)}}
     + \underbrace{
         \mathcal{O}\!\left(
           \frac{\mu^2 L^3\,q^3}
                {\sqrt{T}}
         \right)
       }_{\text{\normalfont smoothing--variance cross-term}
         \;\to\; 0}
     + \underbrace{
         \mathcal{O}\!\left(
           \frac{\sigma_\xi^2\,
                 q}
                {T^{1/2}\mu^2}
         \right)
       }_{\text{\normalfont residual noise}\;\to\; 0}.
\end{align}

For fixed $\mu, L, q, \sigma_\xi$,
the middle cross-term is $\mathcal{O}(1/\sqrt{T})$ and is
therefore absorbed into the leading
$\mathcal{O}(1/\sqrt{T})$ optimization term.
The final convergence bound is:
\begin{equation}
\boxed{
\frac{1}{T}\sum_{t=0}^{T-1}
  \E\!\left[
    \|\nabla_{k(t)}\mathcal{L}(\theta_t)\|^2
  \right]
  \leq
  \mathcal{O}\!\left(\frac{1}{\sqrt{T}}\right)
  + \mathcal{O}\!\left(
      \mu^4 L^2\,q^2
      + \frac{\sigma_\xi^2\,q}
            {T^{1/2}\mu^2}
    \right).
}
\end{equation}

\medskip
\textbf{Consequences.}
\begin{enumerate}
\item The bound depends on
  $q = \dout + \din$
  but \textbf{not on the total rank} $r$, confirming
  that topology-aware scaling removes rank-induced signal
  degradation.
\item The persistent bias
  $\mathcal{O}(\mu^4 L^2 q^2)$ is
  controlled by the smoothing parameter $\mu$.
  Choosing $\mu = \mathcal{O}(q^{-1/2})$
  makes this $\mathcal{O}(L^2 / q^0)
  = \mathcal{O}(L^2)$.
\item The residual noise term
  $\mathcal{O}(\sigma_\xi^2\,
  q / (T^{1/2} \mu^2))$
  vanishes as $T \to \infty$.
\item As $T \to \infty$ and $\mu \to 0$ at an
  appropriate rate, the bound reduces to the standard
  non-convex ZO convergence rate
  $\mathcal{O}(1/\sqrt{T})$.
\end{enumerate}
\end{proof}

\subsection{Proof of Corollary~\ref{cor:coverage_cost}: Coverage Cost for Full-Adapter Stationarity}

\begin{proof}
The atom factor-coordinate blocks are disjoint in the coordinates of
$(\mB,\mA)$, so full-adapter stationarity is measured by the aggregate
quantity
$\sum_{k=1}^{r}\|\nabla_k\mathcal L(\theta)\|^2$ already used in
the statement.
If $k(t)$ is sampled uniformly from $\{1,\ldots,r\}$, taking
expectation over the active atom gives
\[
\E_{k(t)}\!\left[\|\nabla_{k(t)}\mathcal L(\theta)\|^2\right]=\frac{1}{r}\sum_{k=1}^{r}\|\nabla_k\mathcal L(\theta)\|^2 .
\]
For a deterministic cyclic schedule, the same identity holds after
averaging over one full cycle.  Thus full-adapter stationarity
requires atom coverage, but this coverage factor is separate from the
single-query ZO estimator dimension.
\end{proof}

\begin{remark}[Rank and coverage]
Theorems~\ref{thm:directional_collapse} and~\ref{thm:signal_restoration}
form a complete theoretical picture:
\begin{itemize}
\item Theorem~\ref{thm:directional_collapse} shows that
  under naive scaling, the ZO optimization quality
  \emph{collapses} as rank increases.
\item Theorem~\ref{thm:signal_restoration} shows that
  topology-aware scaling $\gamma = \alpha r$ \emph{restores}
  the signal to a rank-invariant level, yielding a
  finite-difference measurement whose active FD-SNR no longer decays with rank.
\item The convergence part of Theorem~\ref{thm:signal_restoration}
  places the restored signal inside a projected-gradient bound governed
  by the active dimension, the smoothing parameter $\mu$, and the noise
  level $\sigma_\xi$.
\end{itemize}
Consequently, increasing rank no longer directly degrades the
active finite-difference measurement once the atom is queried.
Full-adapter stationarity still depends on the atom schedule and
coverage over a cycle, which is why the main text states the gain as a
per-query and per-cycle estimator-budget improvement rather than a
worst-case global convergence dominance claim.
\end{remark}

\subsection{Proof of Corollary~\ref{cor:generalization}: Rank-\texorpdfstring{$m$}{m} Blocks and Other Factorizations}

The corollary uses $\alpha$ as the target coefficient for the active
rank-$m$ block.  Under this block-level convention, the active block
contains $m$ atoms and the compensation rule becomes
$\gamma=\alpha r/m$.  The AR1-ZO specialization corresponds to
$m=1$, which recovers $\gamma=\alpha r$.

\begin{proof}[Derivation]
\textbf{Step 1: Block-$m$ forward pass.}

Consider the factored update partitioned into
$\lceil r/m \rceil$ blocks, each containing $m$ rank-$1$
atoms.  The full update to the weight matrix is:
\begin{equation}
\Delta \mW=\frac{\gamma}{r}\sum_{k=1}^{r}\vb_k\va_k^\top.
\end{equation}
At step $t$, let $B_t$ be the active block of $m$ atoms.
Only the atoms in $B_t$ are perturbed and updated.  The
effective contribution of the active block to the
forward pass is:
\begin{equation}
\Delta \mW_{\text{\normalfont active}}=\frac{\gamma}{r}\sum_{k \in B_t}\vb_k\va_k^\top.
\end{equation}

\medskip
\textbf{Step 2: Explicit block gradient via Lemma 1.}

By applying Lemma~\ref{lemma:active_grad} to each atom
$k \in B_t$, the active-block gradient for
the entire block is the concatenation:
\begin{equation}
\nabla_{B_t}\mathcal{L}
  = \frac{\gamma}{r}\,
    \bigoplus_{k \in B_t}
    \begin{pmatrix}
      \mG\va_k \\ \mG^\top\vb_k
    \end{pmatrix}.
\end{equation}
Its squared norm is:
\begin{equation}
\|\nabla_{B_t}\mathcal{L}\|^2
  = \left(\frac{\gamma}{r}\right)^{\!2}
    \sum_{k \in B_t}
    \big(\|\mG\va_k\|^2
      + \|\mG^\top\vb_k\|^2\big).
\end{equation}
The summation term depends only on $\mG$ and the atoms
in $B_t$; all dependence on $\gamma$ and $r$
is isolated in the prefactor $(\gamma/r)^2$.

\medskip
\textbf{Step 3: Signal restoration condition.}

The key observation is that each active atom contributes
to the forward pass with coefficient $\gamma/r$.  Under
naive scaling ($\gamma = \alpha$), the per-atom
coefficient is $\alpha/r$, so the block of $m$ atoms
has total effective coefficient $m \cdot \alpha/r$.
In a joint update of all $r$ atoms simultaneously
(the full-rank reference), the total effective
coefficient is $r \cdot \alpha/r = \alpha$.

To match the full-rank reference signal, we require
the block's total effective coefficient to equal
$\alpha$:
\begin{equation}
\frac{\gamma}{r}\,m=\alpha\quad\Longrightarrow\quad \gamma=\frac{\alpha r}{m}.
\end{equation}

\medskip
\textbf{Step 4: Verification.}

Substituting $\gamma = \alpha r / m$ into the
explicit formula from Step~2:
\begin{equation}
\|\nabla_{B_t}\mathcal{L}\|^2
  = \left(\frac{\alpha r/m}{r}\right)^{\!2}
    \sum_{k \in B_t}
    \big(\|\mG\va_k\|^2
      + \|\mG^\top\vb_k\|^2\big)
  = \frac{\alpha^2}{m^2}
    \sum_{k \in B_t}
    \big(\|\mG\va_k\|^2
      + \|\mG^\top\vb_k\|^2\big),
\end{equation}
which is independent of the ambient LoRA rank $r$ for a fixed block
size $m$.  Equivalently, the total active block coefficient is
$(\gamma/r)m=\alpha$, so the block-level finite-difference signal does
not decay as the adapter rank grows.  Setting $m=1$ gives the
rank-$1$ AR1-ZO rule $\gamma=\alpha r$.
\end{proof}

\begin{remark}[Scope of tensorized extensions]
The same compensation logic applies to CP, Tucker, or other
multilinear adapters once the joint-normalized parameterization and
the sequential active block are specified.  Adapter-specific proofs
instantiate the corresponding complete block, tangent space,
normalization, and active scale, but the topology-scaling principle
is the same.
\end{remark}

\subsection{Proof of Theorem~\ref{thm:structural_advantage}:
  Conditional Progress Advantage of Aligned Rank-\texorpdfstring{$1$}{1}
  Atoms}

We first state the additional assumptions required
beyond Assumptions~\ref{assum:smoothness}--\ref{assum:unitnorm}.

\begin{assumption}[Spectral Concentration]
\label{assum:spectral}
The dense matrix gradient
$\mG = \nabla_{\mW}\mathcal{L}
  \in \R^{\dout \times \din}$
has SVD $\mG = \sum_i \sigma_i
  \vu_i\vv_i^\top$
with spectral concentration parameter
$\rho = \sigma_1^2 / \|\mG\|_F^2 \in (0,1)$.
\end{assumption}

\begin{assumption}[Atom Alignment]
\label{assum:alignment}
The active rank-$1$ atom $(\vb, \va)$
has alignment
$\beta = \cos^2(\vb, \vu_1)\,
  \cos^2(\va, \vv_1) > 0$
with the top singular pair
$(\vu_1, \vv_1)$ of $\mG$.
\end{assumption}

\textbf{Theorem~\ref{thm:structural_advantage}
(Restated).}
\textit{Under
Assumptions~\ref{assum:smoothness}, \ref{assum:noise},
\ref{assum:unitnorm}, \ref{assum:spectral}, \ref{assum:alignment},
the one-step expected loss decrease admits the active-energy bound}
\begin{equation}
\E\!\left[-\Delta\mathcal{L}\right]
\;\gtrsim\;
\eta\!\left(\frac{\gamma}{r}\right)^{\!2}\|\vg_k\|^2
- \frac{\eta^2 L}{2}\,
\E\!\left[\|\hgrad_k\mathcal L\|^2\right]
- \eta\,\mathcal{O}(\mu^4 L^2\,q^2),
\end{equation}
\textit{and the spectral lower bound}
\begin{equation}
\|\vg_k\|^2
\;\geq\;
\rho\beta\,\|\mG\|_F^2\bigl(\|\va_k\|^2+\|\vb_k\|^2\bigr).
\end{equation}
\textit{Assumption~\ref{assum:unitnorm} guarantees
$\|\va_{k(t)}\|^2+\|\vb_{k(t)}\|^2=\Theta(1)$ on the queried atom,
so the active query carries a $\rho\beta$-scaled fraction of the
dense gradient energy at active dimension $q$.}

\begin{proof}

\textbf{Step 1: Active-atom gradient via SVD
decomposition.}

By Lemma~\ref{lemma:active_grad}, the active-atom
gradient for atom $k$ is:
\begin{equation}
\nabla_{\vb}\mathcal{L}=\frac{\gamma}{r}\,\mG\va_k,\qquad \nabla_{\va}\mathcal{L}=\frac{\gamma}{r}\,\mG^\top\vb_k.
\end{equation}
Expanding $\mG$ via its SVD
$\mG = \sum_i \sigma_i\,
  \vu_i\vv_i^\top$:
\begin{equation}
\mG\va_k=\sum_i \sigma_i(\vv_i^\top\va_k)\vu_i,\qquad \mG^\top\vb_k=\sum_i \sigma_i(\vu_i^\top\vb_k)\vv_i.
\end{equation}

\medskip
\textbf{Step 2: Gradient energy decomposition.}

The squared norms of the individual gradients are:
\begin{equation}
\|\mG\va_k\|^2=\sum_i\sigma_i^2(\vv_i^\top\va_k)^2,\qquad \|\mG^\top\vb_k\|^2=\sum_i\sigma_i^2(\vu_i^\top\vb_k)^2,
\end{equation}
where we used the orthonormality of
$\{\vu_i\}$ and $\{\vv_i\}$,
respectively.

The total active-atom gradient energy is:
\begin{equation}
\label{eq:grad_energy}
\|\nabla_k\mathcal{L}\|^2=\left(\frac{\gamma}{r}\right)^{\!2}\big(\|\mG\va_k\|^2+\|\mG^\top\vb_k\|^2\big).
\end{equation}

\medskip
\textbf{Step 3: Lower bound via spectral
concentration and alignment.}

We lower-bound each term using
Assumptions~\ref{assum:spectral}
and~\ref{assum:alignment}.

For the first term, retaining only the $i=1$
contribution:
\begin{equation}
\|\mG\va_k\|^2=\sum_i\sigma_i^2(\vv_i^\top\va_k)^2\geq\sigma_1^2(\vv_1^\top\va_k)^2=\sigma_1^2\cos^2(\va_k,\vv_1)\|\va_k\|^2.
\end{equation}
(The last equality uses
$(\vv_1^\top\va_k)^2
  = \cos^2(\va_k, \vv_1)\,
    \|\va_k\|^2$
for general (non-unit) $\va_k$.  The factor norms
$\|\va_k\|,\|\vb_k\|$ are kept explicit; under
Assumption~\ref{assum:unitnorm} they are $\Theta(1)$ on the queried
atom, so the lower bound has the form $\Theta(1)\cdot\rho\beta\|\mG\|_F^2$
in $r$.)

By Assumption~\ref{assum:spectral},
$\sigma_1^2 \geq \rho\,\|\mG\|_F^2$,
so:
\begin{equation}
\|\mG\va_k\|^2\geq \rho\,\|\mG\|_F^2\cos^2(\va_k,\vv_1)\|\va_k\|^2.
\end{equation}
Similarly:
\begin{equation}
\|\mG^\top\vb_k\|^2\geq \rho\,\|\mG\|_F^2\cos^2(\vb_k,\vu_1)\|\vb_k\|^2.
\end{equation}

Define
$c_a = \cos^2(\va_k, \vv_1)$,
$c_b = \cos^2(\vb_k, \vu_1)$,
and $\beta = c_a \cdot c_b$
(Assumption~\ref{assum:alignment}).
Combining:
\begin{equation}
\label{eq:energy_lb}
\|\mG\va_k\|^2+\|\mG^\top\vb_k\|^2\geq \rho\,\|\mG\|_F^2(c_a\|\va_k\|^2+c_b\|\vb_k\|^2).
\end{equation}
Since $\beta=c_ac_b$ and $c_a,c_b\in[0,1]$, we have
$\beta\le c_a$ and $\beta\le c_b$.  Hence
\begin{equation}
\|\mG\va_k\|^2+\|\mG^\top\vb_k\|^2\geq \rho\beta\,\|\mG\|_F^2\left(\|\va_k\|^2+\|\vb_k\|^2\right).
\end{equation}

Substituting into~\eqref{eq:grad_energy}:
\begin{equation}
\label{eq:grad_energy_lb}
\|\nabla_k\mathcal{L}\|^2\gtrsim \left(\frac{\gamma}{r}\right)^{\!2}\rho\beta\,\|\mG\|_F^2\left(\|\va_k\|^2+\|\vb_k\|^2\right).
\end{equation}

\medskip
\textbf{Step 4: Per-step descent via the Descent
Lemma.}

By the descent inequality established in the proof
of Theorem~\ref{thm:signal_restoration}
(eq.~\eqref{eq:descent_expect}), the expected loss
decrease satisfies:
\begin{equation}
\E\!\left[-\Delta\mathcal{L}\right]
  = \E\!\left[\mathcal{L}(\theta_t)
      - \mathcal{L}(\theta_{t+1})\right]
  \geq \eta\,\langle
      \nabla_{k(t)}\mathcal{L},\,
      \E\!\left[\hat{\vg}_t\right]
    \rangle
  - \frac{\eta^2 L}{2}\,
    \E\!\left[\|\hat{\vg}_t\|^2\right].
\end{equation}

From the bias bound in
Lemma~\ref{lemma:zo_properties}:
\begin{equation}
\langle
  \nabla_{k(t)}\mathcal{L},\,
  \E\!\left[\hat{\vg}_t\right]
\rangle
  = \|\nabla_{k(t)}\mathcal{L}\|^2
    + \langle
        \nabla_{k(t)}\mathcal{L},\,
        b_t
      \rangle
  \geq \|\nabla_{k(t)}\mathcal{L}\|^2
    - \|\nabla_{k(t)}\mathcal{L}\|\,
      \mathcal{O}(\mu^2 L\,q).
\end{equation}

Therefore:
\begin{equation}
\E\!\left[-\Delta\mathcal{L}\right]
  \geq \eta\,
    \|\nabla_{k(t)}\mathcal{L}\|^2
  - \eta\,\mathcal{O}(\mu^2 L\,q)\,
    \|\nabla_{k(t)}\mathcal{L}\|
  - \frac{\eta^2 L}{2}\,
    \E\!\left[\|\hat{\vg}_t\|^2\right].
\end{equation}

We handle the cross-term using Young's inequality
($ab \leq \frac{a^2}{4\epsilon} + \epsilon b^2$
with $\epsilon = 1$):
\begin{equation}
\eta C_b\mu^2Lq\|\nabla_{k(t)}\mathcal{L}\|\leq\frac{\eta}{4}\|\nabla_{k(t)}\mathcal{L}\|^2+\eta C_b^2\mu^4L^2q^2.
\end{equation}
Absorbing the $\frac{\eta}{4}\|\nabla\|^2$ term
into the leading descent:
\begin{equation}
\E\!\left[-\Delta\mathcal{L}\right]
  \geq \frac{3\eta}{4}\,
    \|\nabla_{k(t)}\mathcal{L}\|^2
  - \eta\,C_b^2\mu^4 L^2\,q^2
  - \frac{\eta^2 L}{2}\,
    \E\!\left[\|\hat{\vg}_t\|^2\right].
\end{equation}
Substituting the lower
bound~\eqref{eq:grad_energy_lb} into the leading
term:
\begin{equation}
\boxed{
\E\!\left[-\Delta\mathcal{L}\right]
  \gtrsim \eta\!\left(\frac{\gamma}{r}\right)^{\!2}
    \|\vg_k\|^2
  - \frac{\eta^2 L}{2}\,
    \E\!\left[\|\hat{\vg}_t\|^2\right]
  - \eta\,\mathcal{O}(\mu^4 L^2\,q^2).
}
\end{equation}
The lower bound on $\|\vg_k\|^2$ from
Step~3 gives the sufficient spectral-alignment version stated in
the main text.
\end{proof}

\begin{remark}[Comparison with generic random blocks]
For a \emph{generic} coordinate block $B$ of the same
dimension $q = \dout+\din$,
a standard random-block averaging calculation gives:
\begin{equation}
\E_B\!\left[\|\nabla_B\mathcal{L}\|^2\right]=\frac{q}{\dout\din}\|\mG\|_F^2.
\end{equation}
The advantage ratio of the aligned rank-$1$ atom over
a generic block is therefore:
\begin{equation}
\frac{\rho\beta\,\|\mG\|_F^2\,
  (\|\va\|^2+\|\vb\|^2)}
{\frac{q}{\dout\din}
  \,\|\mG\|_F^2}
= \rho\beta\,
  \frac{\dout\din}
       {\dout+\din}\,
  (\|\va\|^2+\|\vb\|^2).
\end{equation}
Under Assumption~\ref{assum:unitnorm} (active factor norms
$\|\vb_{k(t)}\|,\|\va_{k(t)}\|=\Theta(1)$) and $\dout=\din=d$, the
advantage ratio simplifies to
\[
\frac{\rho\beta\,\|\mG\|_F^2(\|\va_{k(t)}\|^2+\|\vb_{k(t)}\|^2)}
     {(q/(\dout\din))\,\|\mG\|_F^2}
=\rho\beta\,\frac{\dout\din}{q}\,(\|\va_{k(t)}\|^2+\|\vb_{k(t)}\|^2)
=\Theta(\rho\beta\cdot d).
\]
At $d=4096$, $\rho=0.1$, $\beta=0.05$ with bounded factor norms, the
sufficient lower-bound ratio is
$\Theta(0.1\cdot 0.05\cdot 4096)\approx\Theta(20)$.
\end{remark}


\newpage

\end{document}